%% file: main.tex
\documentclass[10pt]{article} 

\usepackage[preprint]{tmlr}

\input{math_commands.tex}

\input{header.tex}

\usepackage{authblk}
\usepackage{dashrule}

\usepackage{tikz}
\usetikzlibrary{decorations.markings}
\newcommand{\hdottedrule}[3][0]{%
	\tikz[baseline]{\path[decoration={markings,
			mark=between positions 0 and 1 step 2*#3
			with {\node[fill, circle, minimum width=#3, inner sep=0pt, anchor=south west] {};}},postaction={decorate}]  (0,#1) -- ++(#2,0);}}

\usepackage{hyperref}
\usepackage{url}
\usepackage{cleveref}
\usepackage{etoolbox}

\AtBeginEnvironment{thm}{\crefalias{thm}{thm}}
\AtBeginEnvironment{lem}{\crefalias{thm}{lem}}
\AtBeginEnvironment{asm}{\crefalias{thm}{asm}}
\AtBeginEnvironment{cor}{\crefalias{thm}{cor}}

\usepackage{bbm}
\usepackage{algorithm}
\usepackage{cuted}
\usepackage{breqn}
\usepackage{booktabs}       %
\usepackage{nicefrac}       %
\usepackage{microtype}      %

\usepackage{amsfonts}       
\usepackage{nicefrac}       
\usepackage{microtype}      
\usepackage{xcolor}         
\usepackage{subcaption}
\usepackage{svg}
\usepackage{wrapfig}
\usepackage{graphicx}
\usepackage{lipsum}
\usepackage{mathtools}

\definecolor{blue}{HTML}{4d71a6}
\definecolor{green}{HTML}{2e7647}
\definecolor{brown}{HTML}{6d5959}
\definecolor{orange}{HTML}{DE9102}
\definecolor{red}{HTML}{ff4e33}

\definecolor{sc1}{HTML}{aec0da}
\definecolor{sc2}{HTML}{7d9ac4}
\definecolor{sc3}{HTML}{4f75ac}
\definecolor{sc4}{HTML}{38547b}

\crefname{section}{Section}{Sections}
\crefname{appendix}{Supplementary}{Supplementary}
\crefname{figure}{Figure}{Figures}
\crefname{thm}{Theorem}{Theorems}
\crefname{lem}{Lemma}{Lemmas}
\crefname{asm}{Assumption}{Assumptions}
\crefname{cor}{Corollary}{Corollaries}

\newcommand{\envsolo}{\textsc{Solo12} }

\newif\ifcomments
\commentstrue 

\newcommand{\diversity}{\mathrm{Diversity}}

\definecolor{darkblue}{HTML}{228BBB}
\definecolor{darkorange}{HTML}{E39F40}
\newif\ifmcolor
\mcolortrue 

\definecolor{purple}{RGB}{128,0,128}

\usepackage{hyperref}
\hypersetup{
	colorlinks=true,
	linkcolor=blue,
	filecolor=magenta,      
	urlcolor=cyan,
	citecolor=blue
}

\title{Dual-Force: Enhanced Offline Diversity Maximization\\ under Imitation Constraints}

\author[1]{\textbf{Pavel Kolev}}
\author[1,2]{\textbf{Marin Vlastelica}}
\author[1]{\textbf{Georg Martius}}

\affil[1]{University of Tübingen and Tübingen AI Center, Germany}
\affil[2]{ETH Zürich, Switzerland}

\def\month{07}  
\def\year{2016} 

\begin{document}

\maketitle

\begin{abstract}
    Offline diversity maximization under imitation constraints can transform demonstration data into a set of distinct behavioral policies, improving robustness to distribution shift without additional environment interaction.
    In practice, however, existing offline approaches often rely on mutual-information objectives that require training a skill discriminator and can become unstable under the non-stationary rewards induced by alternating Lagrangian optimization.
    We introduce \textsc{Dual-Force}, an offline algorithm that (i) maximizes diversity using an off-policy estimator of a Van der Waals (VdW) force objective computed from successor features, eliminating the skill discriminator, and (ii) stabilizes training under non-stationary intrinsic rewards by conditioning the value function and policy on a pre-trained Functional Reward Encoding (FRE).
    The FRE code also enables zero-shot recall of every encountered skill via its associated latent representation, removing the need to pre-specify a fixed number of skills.
    On two \envsolo simulation benchmarks (locomotion and obstacle navigation), \textsc{Dual-Force} recovers diverse high-performing behaviors while matching a target expert \emph{state} occupancy and improves robustness in adversarial obstacle variations.\footnote{Project website with videos: \href{https://sites.google.com/view/dual-force/home}{https://tinyurl.com/dual-force}}
\end{abstract}

\section{Introduction}\label{sec:intro}

Learning from demonstration data has become a cornerstone of large-scale learning systems, largely due to the sheer volume of available data from sources like videos and robots. 
However, simply mimicking demonstrations is not enough. 
There are several critical reasons why we need to move beyond naive imitation. 
First, demonstrations are often not ``ego-centric'', meaning their state space might not directly align with the agent's, requiring careful adaptation. 
Second, agents often cannot perfectly replicate demonstrations due to their own physical or computational limitations~\citep{li2023learning}.
This highlights the need for agents to extract diverse behaviors that are adaptable to their capabilities, rather than just exact copies of the demonstrations \citep{vlastelica2024doi}. 
Furthermore, robustness to distribution shifts is crucial. 
Since tasks can be solved in multiple ways, extracting diverse policies allows us to identify and select the most robust alternatives \citep{vlastelica2024doi}.
While encouraging risk-averse behavior presents an orthogonal approach to robustness \citep{vlastelica22razer}, our focus is on leveraging diversity to unlock more adaptable and resilient learning from demonstrations.

Maximizing diversity under various constraints is a critical challenge in Reinforcement Learning (RL), frequently addressed within the Constrained Markov Decision Process (CMDP) framework \citep{zahavy2023discovering,vlastelica2024doi,cheng2024dominic}. 
These approaches typically solve the underlying constrained optimization problem using Lagrangian relaxation, employing a two-phase alternating scheme where constraints are incorporated into the reward objective and scaled by adaptive Lagrangian multipliers to reduce violations \citep{zahavy2023discovering,cheng2024dominic}. 
While effective in online settings \citep{zahavy2023discovering,cheng2024dominic}, applying these techniques to offline learning presents unique challenges. 
Our work specifically targets the offline setting, where direct environment interaction is absent. 
Crucially, instead of enforcing near-optimal returns, we impose imitation-based constraints, a strategy previously explored by~\citet{vlastelica2024doi}.
Our method for offline learning from demonstrations builds upon the powerful Fenchel duality theory, specifically adapted for RL within the DIstribution Correction Estimation (DICE) framework \citep{nachum2020reinforcement,kim2022demodice,ma2022smodice,maYJB22goalbasedSmodice}. 
This foundation allows us to address the complexities of diversity maximization in real-world, data-constrained environments.

\begin{figure}
	\centering
	\includegraphics[width=0.85\linewidth]{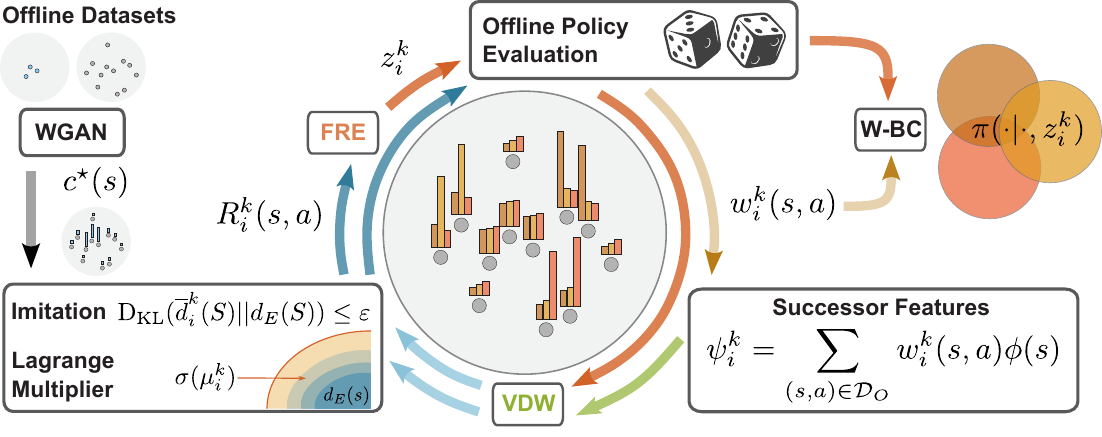}
	\caption{Illustration of \textsc{Dual-Force}. The pseudocode is presented in \Cref{alg:dual-force}.} 
	\label{fig:overview}
\end{figure}

Prior approaches to learning diverse skills in RL, such as the method proposed by \citet{vlastelica2024doi}, face significant practical challenges when applied in offline settings. 
While their work introduces a diversity objective based on a variational lower bound on mutual information (which necessitates a skill-discriminator), this design choice introduces several complications for practitioners.
Specifically, the reliance on a skill-discriminator leads to training instability: updating the policy and skill discriminator in a single step within an offline setting provides less accurate policy estimates compared to online Monte Carlo sampling \citep{eysenbach19diayn}. 
This inaccuracy, combined with a non-stationary reward (due to Lagrange multipliers and the skill-discriminator), often results in a discriminator that struggles to accurately differentiate between skills. 
Although an additional information gain term \citep{strouse2021disdain} can mitigate this, its effect quickly diminishes in offline scenarios. 
\citet{vlastelica2024doi} acknowledge this difficulty in training the skill discriminator as a major practical hurdle.
Furthermore, their DOI algorithm violates the stationary reward assumption within the DICE framework, which can lead to unstable training of the value function.
Moreover, a common limitation in existing skill-discovery methods, including those by \citep{zahavy2023discovering} and \citep{vlastelica2024doi}, is the requirement to predefine the number of skills. 
This not only limits the adaptability of the learned policy, often resulting in the forgetting of previously acquired skills in favor of more recent ones, but also introduces a runtime complexity that scales linearly with the number of skills.

In this work, we introduce \textsc{Dual-Force}, a novel offline algorithm engineered to overcome the limitations of previous approaches. 
Our core innovation lies in developing an off-policy evaluation procedure for a physically inspired diversity objective. 
Specifically, we leverage the Van der Waals (VdW) force \citep{zahavy2023discovering} to achieve enhanced diversity. 
This crucial design choice eliminates the need to learn a separate skill discriminator, a significant practical challenge in prior methods like \citet{vlastelica2024doi}, while providing a robust diversity signal even in the challenging offline setting. 

\textbf{Key idea.} \textsc{Dual-Force} makes the VdW diversity signal practical in offline imitation-constrained skill discovery by estimating, off-policy, all quantities needed for the VdW force (successor features and dual-conjugate variables) using a DICE-style importance sampling procedure.
To address the non-stationary rewards created by alternating Lagrangian optimization, we condition the value function and policy on a pre-trained Functional Reward Encoding (FRE)~\citep{frans24fre}, which stabilizes learning and allows us to \emph{index and recall} every reward/skill encountered during training via its latent code.

\newpage
\textbf{Contributions.}
\begin{itemize}[topsep=0pt, itemsep=0.25pt, parsep=0.25pt]
  \item \textbf{Discriminator-free offline diversity:} an off-policy estimator for the VdW-force objective that removes the skill discriminator used by mutual-information objectives, improving stability in the offline setting.
  \item \textbf{Non-stationarity handling + recall:} FRE-conditioned value/policy learning that mitigates moving-target instability and enables zero-shot recall of skills via stored latent codes, removing dependence on a fixed ``number of skills''.
  \item \textbf{Robotics evaluation:} empirical results on \envsolo locomotion and obstacle navigation demonstrating diverse behaviors under state-only imitation constraints and improved robustness under obstacle perturbations.
\end{itemize}

\paragraph{Roadmap.}
The remainder of the paper is organized as follows. We review background in \Cref{sec:prel}, formalize the objective and constraints in \Cref{sec:problemFormulation}, present our method in \Cref{sec:method}, and evaluate the learned skills set on benchmarks in \Cref{sec:experiments}. We then discuss related work in \Cref{sec:rel-work} and conclude in \Cref{sec:conc}.

\section{Preliminaries}\label{sec:prel}
We frame RL problems using the framework of Markov Decision Processes (MDPs)~\citep{puterman2014markov}.
An MDP is defined by the tuple
$(\mathcal{S}, \mathcal{A}, \mathcal{R}, \mathcal{P}, \rho_0, \gamma)$, representing the state space $\mathcal{S}$, action space $\mathcal{A}$, a reward function $\mathcal{R}$, a state transition probability function $\mathcal{P}$, an initial state distribution $\rho_0$, and a discount factor  $\gamma$. 
A crucial concept for understanding RL objectives is the state-action occupancy measure, $d_{\pi}(s,a)$, which quantifies the discounted expected number of times a particular state-action pair $(s,a)$ is visited when following policy $\pi$. 
This allows us to reframe the standard RL objective as maximizing the inner product between the occupancy measure and the reward function \citep{ZahavyODS21}, namely $\max_{d_\pi \in \mathcal{K}}\innerproduct{d_{\pi}}{r}$.
In this work, we specifically consider a diversity objective, which takes as input $n$ state-action occupancies $(d_1,\dots,d_n)$, each induced by a distinct policy $\pi_i$.

\noindent\textbf{Offline data access.}
We only assume \emph{state-only} expert samples $\mathcal{D}_E\sim d_E(S)$ (no expert actions) and a separate offline dataset $\mathcal{D}_O\sim d_O(S,A)$ collected by a mixture of behaviors.
Unless stated otherwise, expectations over $d_O$ and $d_E$ are implemented as empirical averages over $\mathcal{D}_O$ and $\mathcal{D}_E$, respectively.

\subsection{Constrained Markov Decision Process (CMDP)}

A critical challenge in RL is to generate not just a single optimal policy, but a diverse set of high-performing policies.
Let $v_e^* := \langle d_{\pi_e^*}, r_e \rangle$ denote the extrinsic value of an expert policy $\pi_e^*$. 
\citet{zahavy2023discovering} explored a CMDP formulation for computing a set of policies $\Pi^{n}=\{\pi_{i}\}_{i=1}^{n}$ that simultaneously maximize a diversity objective while ensuring that each policy maintains a minimum level of extrinsic performance:
\begin{equation}\label{eq:constrained_mdp}
	\max_{\Pi^n} \ \diversity(\Pi^n) \ \text{subject to} \ \innerproduct{d_{\pi}}{r_e} \geq \alpha v^*_e,\quad \forall \pi \in \Pi^n,
\end{equation}
Here, $\alpha \in(0, 1]$ controls the required fraction of expert performance.
To handle diversity objectives, \citet{zahavy2023discovering} used an iterative optimization procedure that reduced the problem to solving a sequence of standard RL problems.
Each problem leverages an intrinsic reward $r_i^{k+1} = \nabla_{d_i}\diversity(\overline{d}_{1}^{k}, \dots, \overline{d}_{n}^{k})$, which is the gradient of the diversity objective evaluated at the time-averaged state-action occupancies $\overline{d}_{i}^{k}$ from the previous step.

To manage the interplay between extrinsic performance and diversity, a common strategy is to employ Lagrange relaxation.
This transforms the constrained optimization problem into an unconstrained RL problem where the reward function $r^{k+1} = r_e + \lambda_i r_i^{k+1}$ is a combination of the extrinsic reward and the intrinsic term scaled by a Lagrange multiplier $\lambda\geq0$
that balances the two rewards.
The Lagrange multipliers are dynamically adjusted to minimize a loss function
$\mathcal{L}_{\lambda}=\sum_{i=1}^{n}\lambda_{i}(\langle d_{i},r_{e}\rangle - \alpha v_{e}^{*})$ that penalizes constraint violations. 
Crucially, a Lagrange multiplier increases when its associated constraint is violated and decreases otherwise, providing an adaptive mechanism to balance the competing objectives. Practical implementations, such as those by \citet{StookeAA20} and \citet{cheng2024dominic}, often incorporate bounded Lagrange multipliers using a sigmoid activation.

\subsection{Functional Reward Encoding (FRE)}

The concept of Functional Reward Encoding (FRE), recently proposed by \citet{frans24fre}, offers a powerful paradigm for handling and generalizing across diverse reward functions.
Building on information bottleneck methods \citep{tishby00InfoBotneck,alemi17infoBotneck}, FRE utilizes a transformer-based variational auto-encoder to encode state-reward samples into a compact latent representation $z_r$.
The key idea behind FRE is that this latent representation, encoded from an arbitrary subset of state-reward samples, should be maximally compressive while remaining highly predictive for decoding rewards from other arbitrary subsets of state-reward samples.
This enables remarkable zero-shot capabilities: a latent-conditioned policy $\pi(\cdot|\cdot,z)$ once trained with a fixed, pre-trained FRE model (on a diverse set of reward functions), can optimize previously unseen downstream reward functions $r$ simply by encoding a few state-reward samples into a latent $z_r$.
Empirical evaluations in standard D4RL offline environments \citep{fu2020d4rl} have shown the effectiveness of FRE in enabling the policy to generalize to new, unseen reward functions without further training.

\section{Problem Formulation}\label{sec:problemFormulation}

We aim to solve the following optimization problem
\begin{eqnarray}\label{eq:main-problem}
\max_{d_{1},\dots,d_{n}}&\diversity(d_{1},\dots,d_{n})&\label{eq:constrained-problem-diversity}\\
\text{subject to}&\Dkl\left(d_{i}(S)||d_{E}(S)\right)\leq\varepsilon&\forall i\in\{1,\dots,n\},\label{eq:constrained-problem-constraints}
\end{eqnarray}
where $d_E(S)$ is a \emph{state-only} expert occupancy (expert actions are not observed).
Compared to prior offline constrained diversity maximization, we (i) use a diversity objective that admits a stable off-policy estimator without training a skill discriminator, and (ii) relax the imitation constraints in a way that preserves their state-only nature while allowing additional freedom for diversity.
This provides greater flexibility in maximizing diversity.

\subsection{Diversity Measures}\label{subsec:DiversityMeasures}

Prior work by \citet{vlastelica2024doi} used a variational lower bound on the mutual information $\mathcal{I}\left(S;Z\right)$ between states and latent skills, leading to the following diversity objective:
\begin{equation}\label{eq:VarLB}
\mathcal{I}\left(S;Z\right)\geq\mathbb{E}_{p(z),d_{z}(s)}\left[\log q(z|s)\right]+\mathcal{H}\left(p(z)\right)=\sum_{z\in Z}\mathbb{E}_{d_{z}(s)}\left[\frac{\log\left(|Z|q(z|s)\right)}{|Z|}\right],
\end{equation}
Here, $q(z\mid s)$ is implemented as a learned skill discriminator; in the offline setting, its training signal can degrade under moving targets and imperfect one-step policy updates, motivating discriminator-free diversity objectives.
In this setting, $p(z)$ is a categorical distribution over a discrete set $Z$ of $|Z|$ distinct indicator vectors in $\mathbb{R}^{|Z|}$ and $d_z(s):=d_{\pi_z}(s)$ is a state occupancy induced by a skill-conditioned policy $\pi_z$.
While their approach allows for off-policy estimation using a DICE importance sampling method, it requires learning a skill-discriminator, which leads to training instability in the offline setting.

In contrast, \citet{zahavy2023discovering} proposed a diversity objective rooted in a distance measure from \citep{abbeel2004apprenticeship}. Their method maximizes the minimum squared $\ell_2$ distance between successor features of different skills
\begin{equation}\label{eq:rep_force}
\max_{d_{1},...,d_{n}}\ \frac{1}{n}\sum_{i=1}^{n}\min_{j\neq i}\norm{\psi^{i}-\psi^{j}}_{2}^{2}.
\end{equation}
Specifically, given a feature mapping $\phi : \mathcal{S}\rightarrow \mathbb{R}^n$, successor features~\citep{dayan1993successor, barreto2017successor} are defined by $\psi_{i}=\mathbb{E}_{d_{i}(s)}[\phi(s)]$.
A significant advantage of this convex objective is that its gradient is readily available in closed form (derived in Lemma~\ref{lem:div_is_conv}), eliminating the need for a skill discriminator.

Furthermore, \citet{zahavy2023discovering} introduced a physically inspired objective based on Van der Waals (VdW) force, formulated as:
\begin{equation}\label{eq:vdw_force}
\max_{d_{1},...,d_{n}}\ 0.5\sum_{i=1}^{n}\ell_{i}^{2}-0.2(\ell_{i}^{5}/\ell_{0}^{3}),
\end{equation}
where $\ell_{i}:=\norm{\psi_{i}-\psi_{j_{i}^{\star}}}_{2}$ and $j_{i}^{\star}:=\argmin_{j\neq i}\norm{\psi_{i}-\psi_{j}}_{2}$.
We adopt this formulation in our work because it allows for the level of diversity to be precisely controlled by the parameter $\ell_0$.
When the successor features are in close proximity $\ell_i<\ell_0$, a repulsive force dominates, pushing them apart. 
Conversely, when $\ell_i>\ell_0$, an attractive force prevails, drawing them closer.
In the limit $\ell_0 \rightarrow \infty$, the formulation in \cref{eq:rep_force} is recovered.

\section{Method}\label{sec:method}

Traditional offline methods for learning diverse skills often face instability and are limited to a fixed number of skills. 
These challenges stem from their reliance on skill discriminators and difficulties in managing non-stationary rewards.
To overcome these, our novel method introduces a robust diversity objective inspired by Van der Waals force, effectively integrated with the DICE framework for efficient off-policy estimation and Functional Reward Encoding (FRE) to handle reward non-stationarity.
This approach guarantees stable training and significantly expands the set of learnable skills, scaling effectively with the number of iterations.

\paragraph{Method overview.}
\textsc{Dual-Force} solves Problem~\ref{eq:main-problem} via an alternating offline optimization loop that (i) computes a discriminator-free diversity signal from successor features, (ii) solves a KL-regularized offline RL subproblem via DICE to obtain occupancy ratios, (iii) trains policies with weighted behavior cloning, and (iv) updates bounded multipliers from an offline constraint-violation estimator. 
Since the reward changes across iterations, we stabilize value/policy learning by conditioning on a pre-trained FRE latent, which also enables recall of skills encountered during training.

\subsection{Van der Waals Force}

In offline diversity maximization under imitation constraints, the main practical obstacle is that both (i) the constraint-violation signal and (ii) many diversity signals (e.g., discriminator-based objectives) are difficult to estimate reliably without environment interaction, which destabilizes alternating optimization. Our key technical insight is that, under KL-divergence imitation constraints (cf.~\cref{eq:constrained-problem-constraints}), all relevant quantities needed for VdW-based diversity optimization, including dual variables and successor features, can be estimated off-policy using a DICE importance sampling procedure from the offline dataset.

\subsubsection{Relaxing the state-only imitation constraint for offline estimation}

For Problem~\ref{eq:main-problem}, we've set our diversity objective using the Van der Waals force as defined in \cref{eq:vdw_force}.
Our initial observation, formalized in Lemma~\ref{lem:state-KL-est}, reveals that the imitation constraints in ~\cref{eq:constrained-problem-constraints} can be effectively relaxed to:
\begin{equation}\label{eq:relaxed_imitation_constraint}
-\mathbb{E}_{d_{i}(s)}\left[\log\frac{d_{E}(s)}{d_{O}(s)}\right]+\Dkl\left(d_{i}(S,A)||d_{O}(S,A)\right)\leq\varepsilon,\qquad\forall i\in\{1,\dots,n\}.
\end{equation}
Crucially, this relaxation preserves the \emph{state-occupancy} nature of the constraint while exposing a state-action KL term that matches the DICE regularizer, enabling a single off-policy estimation pipeline.
This constitutes a tighter relaxation of the imitation constraints, crucial as it preserves the essential state-occupancy nature while still allowing for efficient computation. 
This approach is a significant improvement over enforcing the more restrictive state-action occupancy constraints that are typically tied to a fixed {\color{ourdarkblue} SMODICE-expert}~\citep{vlastelica2024doi}.

\subsubsection{VdW intrinsic reward from successor features (discriminator-free diversity)}

Drawing on similar arguments from \citep{zahavy2023discovering}, we developed an iterative procedure.
In each iteration $k+1$, we consider a Lagrange relaxation of Problem~\ref{eq:main-problem} for the $i^{\mathrm{th}}$ state-action distribution $d_i$:
\begin{equation}\label{eq:iterative-procedure-Lagrange-relaxation}
\min_{\lambda_{i}\geq0}\max_{d_{i}}\mathbb{E}_{d_{i}(s,a)}\left[\beta_{i}^{k}(s,a)\right]+\lambda_{i}\left[\mathbb{E}_{d_{i}(s,a)}\left[\log\frac{d_{E}(s)}{d_{O}(s)}\right]-\Dkl\left(d_{i}(S,A)||d_{O}(S,A)\right)\right].
\end{equation}
Here, $\lambda_i$ acts is a Lagrange multiplier and the term $\beta_{i}^{k}=\nabla_{d_{i}}\diversity(\overline{d}_{1}^{k},\dots,\overline{d}_{n}^{k})$ is a dual-conjugate variable.
Operationally, $\beta_i^k(s,a)$ is the intrinsic reward used in iteration $k\!+\!1$; it is the closed-form gradient of the VdW objective evaluated at the Polyak-averaged occupancies, eliminating the need for a learned skill discriminator.

In our specific context, with the diversity objective set to VdW force \cref{eq:vdw_force}, this simplifies to:
\[
\beta_{i}^{k}(s,a):=(1-(\ell_{i}^{k}/\ell_{0})^{3})\langle\phi(s),\psi_{i}^{k}-\psi_{j_{i}^{\star}}^{k}\rangle,
\]
where $\psi_{i}^{k}:=\mathbb{E}_{\overline{d}_{i}^{k}(s)}[\phi(s)]$, $\ell_{i}^{k}:=\Vert\psi_{i}^{k}-\psi_{j_{i}^{\star}}^{k}\Vert_{2}$ and $j_{i}^{\star}:=\argmin_{j\neq i}\norm{\psi_{i}^{k}-\psi_{j}^{k}}_{2}$ are all defined with respect to a time-averaged state-action distribution $\overline{d}_{i}^k=\frac{1}{t}\sum_{t=1}^{k}d_i^t$.

Next, we apply Fenchel duality to solve offline the inner maximization problem in \cref{eq:iterative-procedure-Lagrange-relaxation}. 
For practical implementation, we use bounded Lagrange multipliers $\sigma(\mu_i)$ and a Polyak update scheme to maintain the time-averaged state-action distributions $\{\overline{d}_{1}^{k},\dots,\overline{d}_{n}^{k}\}$.
Once the bounded Lagrange multipliers are fixed, we arrive at a standard RL problem, which is regularized by a KL-divergence term:
\begin{equation}\label{eq:fix-lmbda-problem}
\max_{d_{i}}\mathbb{E}_{d_{i}(s,a)}\left[R_{i}^{k}(s,a)\right]-\Dkl\left(d_{i}(S,A)||d_{O}(S,A)\right),
\end{equation}
where we abbreviate the iteration-$k$ reward as $R_i^{k} := R_i^{\mu_i^{k}}$ and define:
\begin{equation}\label{eq:vdw-dual-reward}
R_{i}^{k}(s,a):=
\underbrace{ (1-\sigma(\mu_{i}^{k})) }_{\text{Constraint Satisfaction}}\underbrace{\beta_{i}^{k}(s,a)}_{\text{VdW-Diversity}} + \underbrace{ \sigma(\mu_{i}^{k}) }_{\text{Constraint Violation}}\underbrace{\log\frac{c^{\star}(s)}{1-c^{\star}(s)}}_{\text{Expert-Imitation}}.
\end{equation}
Here, $c^{\star}(s)$ denotes a pretrained state-discriminator~\citep{kim2022demodice,ma2022smodice}.
This discriminator is critical for distinguishing states found in an expert dataset $\mathcal{D}_E\sim d_E(S)$ from those in on offline dataset $\mathcal{D}_O\sim d_O(S,A)$, and notably, it satisfies $c^{\star}(s)=d_E(s)/(d_E(s)+d_O(s))$.

\subsection{Offline Estimators using Fenchel Duality}

Given the KL-regularized objective in \cref{eq:fix-lmbda-problem}, the remaining challenge is to solve it \emph{offline} and recover the occupancy correction needed to estimate diversity, constraint violation, and policy updates from $d_O$. We use the DICE dual to obtain (i) a value-function objective trainable from offline transitions and (ii) normalized importance ratios that reweight samples from $d_O$ into the target occupancy.

The DICE framework, extensively explored in~\citep{nachum2020reinforcement,kim2022demodice,ma2022smodice,maYJB22goalbasedSmodice}, offers a principled approach to solving the offline KL-regularized RL problem outlined in \cref{eq:fix-lmbda-problem}.
In particular, framework~\citep{nachum2020reinforcement,kim2022demodice,ma2022smodice,maYJB22goalbasedSmodice} solves offline the KL-regularized RL problem in \cref{eq:fix-lmbda-problem} by considering its dual formulation, which reads
\begin{equation}\label{eq:Vstar}
V_{i}^{\star}=\argmin_{V(s)}(1-\gamma)\mathbb{E}_{s\sim\rho_{0}}\left[V(s)\right]+\log\mathbb{E}_{d_{O}(s,a)}\exp\left\{ R_{i}^{k}(s,a)+\gamma\mathcal{T}V(s,a)-V(s)\right\},
\end{equation}
where we denote by $\mathcal{T}V(s,a):=\mathbb{E}_{\mathcal{P}(s^{\prime}|s,a)}V(s^{\prime})$.
Minimizing \cref{eq:Vstar} yields $V_i^\star$ such that the induced TD residual defines the occupancy correction used to form normalized importance weights.
The temporal difference (TD) error is given by
\[
\delta_{i}(s,a)=R_{i}^{k}(s,a)+\gamma\mathcal{T}V_{i}^{\star}(s,a)-V_{i}^{\star}(s).
\]
Then, the primal solution of Problem~\ref{eq:iterative-procedure-Lagrange-relaxation} reads
\begin{equation}\label{eq:dopt}
\eta_{i}(s,a):=\frac{d_{i}^{\star}(s,a)}{d_{O}(s,a)}=\frac{\exp\{\delta_{i}(s,a)\}}{\mathbb{E}_{d_{O}(s^{\prime},a^{\prime})}\exp\{\delta{}_{i}(s^{\prime},a^{\prime})\}}=:\mathrm{softmax}_{d_{O}(s,a)}\left(\delta_{i}(s,a)\right).
\end{equation}
The ratio $\eta_i(s,a)=d_i^\star(s,a)/d_O(s,a)$ specifies which offline transitions must be upweighted to match the occupancy that optimizes the current non-stationary reward $R_i^k$.
Based on the importance ratios $\eta_i$ we can compute offline all necessary estimators.
In particular, for any function $f$, we can estimate offline the following expectation:
\begin{equation}\label{eq:offline-importance-sampling}
\E_{d_{i}(s,a)}[f(s,a)]=\E_{d_{O}(s,a)}[\eta_{i}(s,a)f(s,a)].
\end{equation}
Using \cref{eq:offline-importance-sampling} we can train offline an optimal policy by maximizing the following weighted behavior cloning objective $\E_{d_{O}(s,a)}[\eta_{i}(s,a)\log\pi_{i}(a\vert s)]$.
Similarly, we can estimate offline the successor features $\psi_i= \E_{d_O(s,a)}[\eta_i(s,a) \phi(s)]$ and also maintain the associated time-averaged successor representations $\psi_{i}^{k}=\mathbb{E}_{d_{O}(s,a)}[\overline{\eta}^{k}(s,a)\phi(s)]$, where $\overline{\eta}^{k}=\frac{1}{t}\sum_{t=1}^{k}\eta_{i}^{t}$.
Thus, the ratios $\eta_i$ support diversity estimation (via $\psi_i$), constraint monitoring, and policy learning, avoiding separate estimation modules.

This gives us the tool to estimate offline the VdW-Diversity term in \cref{eq:vdw-dual-reward}.

Next, we dynamically adjust the bounded Lagrange multipliers $\sigma(\mu_i)$ based on an offline estimation of the corresponding constraint violation.
In Corollary~\ref{cor:kl-estim}, we show that the LHS of \cref{eq:relaxed_imitation_constraint} admits an estimator
\begin{equation}
\mathbb{E}_{d_{O}(s,a)}\left[\eta_{i}(s,a)\left(\log\eta_{i}(s,a)-\log\frac{c^{\star}(s)}{1-c^{\star}(s)}\right)\right].
\end{equation}
In practice, however, we only have access to finitely many samples of the state occupancy $d_{O}(s,a)$.
Thus, in Lemma~\ref{lem:KL-estimator}, we derive the following finite sample estimator of the LHS of \cref{eq:relaxed_imitation_constraint}:
\[
\phi_i:=\log|\mathcal{D}_{O}| \,\, +\sum_{(s,a)\in\mathcal{D}_{O}}w_{i}(s,a)\left[\log w_{i}(s,a)-\log\frac{c^{\star}(s)}{1-c^{\star}(s)}\right],
\]
where
\[
w_{i}(s,a):=\mathrm{softmax}_{\mathcal{D}_{O}}(\delta_{i}(s,a))=\frac{\exp\{\delta_{i}(s,a)\}}{\sum_{(s^{\prime},a^{\prime})\in\mathcal{D}_{O}}\exp\{\delta_{i}(s^{\prime},a^{\prime})\}}.
\]
Furthermore, we can optimize the bounded Lagrange multipliers $\sigma(\mu_i)$ by minimizing the loss
$\mathcal{L}_{\mu}:=\sum_{i=1}^{n}\sigma(\mu_{i})(\varepsilon-\phi_{i})$.
Here we use gradient descent to adapt the multipliers $\mu_i$.

\subsection{Handling Non-Stationary Rewards}

To optimize Problem~\ref{eq:main-problem} offline, we extend the heuristic of \citet{zahavy2023discovering} into an alternating optimization scheme (cf.~\Cref{alg:dual-force}). Unlike standard DICE settings with fixed rewards, our reward $R_i^k$ changes each iteration due to both the diversity term $\beta_i^k$ and the bounded multipliers $\sigma(\mu_i^k)$. In offline training, repeatedly fitting a value function to a moving target, often with only a single gradient step per iteration, can destabilize the DICE optimization, which then corrupts the ratios and collapses the overall alternating procedure.

In this work, we address the preceding challenge by conditioning the value function (and policy) on a latent representation of a Functional Reward Encoding (FRE) \citep{frans24fre}, which is pre-trained on random linear functions, random two-layer neural networks, and simple human-engineered rewards.
Conditioning the value function $V_i(\cdot,z_r)$ and the policy $\pi_i(\cdot\mid\cdot,z_r)$ on the FRE latent turns the non-stationary reward sequence into a family of reward-indexed tasks, enabling parameter sharing without losing reward identifiability across iterations.
Further details on pre-training are given in \cref{app:FRE}.

In each iteration, given a fixed reward $r$ we compute its encoded FRE latent representation $z_r(S)$ over a subset of state-reward samples $L(r,S):=\{(s,r(s)) \,: \, s\in S\}$, where $S$ is subset of states sampled uniformly at random from $\mathrm{States}[\mathcal{D}_O]$.
Further, to reduce the variance, we sample uniformly at random several state subsets $\{S_1, \dots, S_m\}$ and take the mean $z_r$ over their FRE latent representations $z_r(S_i)$.
Storing each encountered latent (e.g., $z_i^k$) provides a lightweight key for recalling the corresponding skill $\pi_i(\cdot\mid\cdot,z_i^k)$ at evaluation time, yielding zero-shot recall of previously encountered rewards.

In line with standard deep learning practices, the value function and policy are parameterized with neural networks and consequently updated with a single gradient step over batches sampled uniformly at random.
Given a batch $\mathcal{B}$, the policy loss becomes
$\frac{|\mathcal{D}_O|}{|\mathcal{B}|}\sum_{(s,a)\in\mathcal{B}}w_{i}^{k+1}(s,a)\log\pi_{i}(a|s, z_i^k)$.

\section{Experiments}\label{sec:experiments}

\noindent\textbf{Evaluation questions.}
We evaluate: (\textbf{Q1}) whether learned skills are diverse in successor-feature space, (\textbf{Q2}) whether they satisfy state-only imitation constraints while achieving high expert-return, and (\textbf{Q3}) whether diversity yields robustness under adversarial obstacle variations.

\textbf{Data Collection.} 
To evaluate our method, we consider the 12 degree-of-freedom quadruped robot \envsolo~\citep{grimminger2020open} on two robotic tasks in simulation: locomotion and obstacle navigation.
For fair comparison and consistency, in terms of quality and diversity of learned skills, our experiments closely follow the setup in \citep{vlastelica2024doi} and use offline datasets whose collection process is described in Section G ``SOLO-12 Dataset Collection'' of their work.
In particular, following \citep{kim2022demodice,ma2022smodice,vlastelica2024doi}, a state discriminator is learned to differentiate between state demonstrations collected by an expert and from different behavioral policies.
To ensure that these behavioral policies provide sufficient diversity while fulfilling a specific task, an online algorithm is run for unsupervised skill discovery subject to value constraints, DOMiNiC~\citep{cheng2024dominic}, and Monte Carlo trajectories are collected from various policy checkpoints throughout the training process.
The expert dataset is then mixed into the offline dataset, and the information about which state-action comes from the expert remains hidden to our algorithm.
This replicates the experimental setup in \citep{vlastelica2024doi} and allows for a direct comparison between their algorithm (DOI) and ours (\textsc{Dual-Force}).

\textbf{Experimental Setup.} 
For each experiment, we train: a state-discriminator $c^{\star}$, a Functional Reward Encoding $\mathcal{F}$, a {\color{ourdarkblue} SMODICE-expert} (to visualize the target imitation behavior), and diverse skills that satisfy the imitation constraint.
To enable recall of all skills encountered during training, we store, for each intrinsic reward $R_i^k$, the corresponding mean FRE latent code $z_i^k$ and condition the learned policy on this code at evaluation.
For both robotic tasks (locomotion and obstacle navigation), the state-discriminator and the {\color{ourdarkblue} SMODICE-expert} are trained on the full state space (see \cref{app:solo12_state_space}), forcing the learned skills to imitate the state-only expert demonstrations in their entirety.
In contrast, for the diversity objective, we also follow \citep{cheng2024dominic, vlastelica2024doi} by using successor features induced by a projection onto the most relevant components of the state space, as this is beneficial in practice.

\textbf{Skills Evaluation.} 
It is important to note that our problem formulation does not assume access to a reward signal in the offline dataset.
However, if the offline dataset contains reward labels, then each skill learned during training can be evaluated off-policy using its corresponding importance ratios $\eta_i$.
In this work, we conduct an online evaluation of each learned skill by rolling out 30 Monte Carlo trajectories in simulation. 
We then compute the mean values of (i) the successor features and (ii) the cumulative return, relative to the reward signal (hidden to our algorithm) used to optimize the expert policy.
After training, we report the subset of learned codes $z_i^k$ whose corresponding policies $\pi(\cdot\mid\cdot,z_i^k)$ achieve at least $50\%$ of the expert’s return under the (withheld) evaluation reward used to train the expert.
While intermediate iterations may temporarily trade off imitation for diversity, the bounded-multiplier updates typically drive them back toward constraint satisfaction, yielding a large fraction of skills that imitate the expert and achieve high reward.

\textbf{Practical Implementation.}
\looseness=-1
For each state-action occupancy $d_i$ in Problem~\ref{eq:main-problem}, we train a value function and a policy, parameterized by neural networks.
We empirically observed that skill diversity increases and the training procedure stabilizes, when 
the neural network weights of the value functions (and similarly the policies) are independent across all state-action occupancies.
This is efficiently implemented by running the forward pass over all value functions (and policies) in parallel.
In the experiments below, we optimize over three state-action occupancies and assign them with the following color map: 
$d_1$ is {\color{ourorange} orange}, $d_2$ is {\color{ourbrown} brown}, and $d_3$ is {\color{ourred} red}.

\textbf{Main Result}. \textsc{Dual-Force} enjoys a strong diversity signal, satisfies the imitation constraint (see \cref{app:diversity_constraint}), and recalls all skills encountered during training, with the number of skills growing with training iterations. 
This is especially important under non-stationary diversity rewards, where fixed placeholder-based approaches forget earlier skills in favor of more recent ones.

\subsection{Locomotion Task}

\textbf{Data Collection.} 
The expert dataset is collected from a uni-modal expert trained to walk straight with constant linear velocity and {\color{cmiddle} middle} base-height.
The offline dataset contains non-expert behaviors achieving constant linear and angular velocity, and walking movements with different base-heights ({\color{clow} low}, {\color{cmiddle} middle}, {\color{ourorange} orange}).

\textbf{Diversity Objective.} 
The state space is projected onto the joint position components represented by a 12-dimensional vector, which captures the movements of each of the four legs: front left, front right, hind left, and hind right, and includes hip abduction/adduction, hip flexion/extension, and knee flexion/extension.
The joint positions serve as a proxy for body height, which is missing from the offline dataset.

\textbf{Results.}
These results answer \textbf{Q1} and partially \textbf{Q2}. 
\Cref{fig:locom-robustness-experiment} shows that \textsc{Dual-Force} recovers multiple distinct locomotion modes from the offline dataset, including all three base-height behaviors. 
At the same time, all learned skills maintain the expert's characteristic constant linear and angular velocity, indicating that diversity is achieved while preserving the state-only imitation target. 
Consistent with this, \Cref{fig:SFs}(a,b) shows that the learned skills form three well-separated clusters in successor-feature space, corresponding to the three base-height modes: within-cluster $\ell_2$ distances are small, while between-cluster distances are large. 
This indicates that the learned skills capture genuinely distinct behaviors rather than minor variations of the same gait. 
We visualize this structure using a 2D UMAP projection of the successor features.

\begin{figure}
    \newlength{\gridspace}
    \setlength{\gridspace}{1pt}
    \newcommand\scale{0.22}
    \setlength{\fboxrule}{1pt}
    \begin{tabular}{cccc}
    \raisebox{7ex}{\color{chigh}\fbox{\includegraphics[scale=0.09]{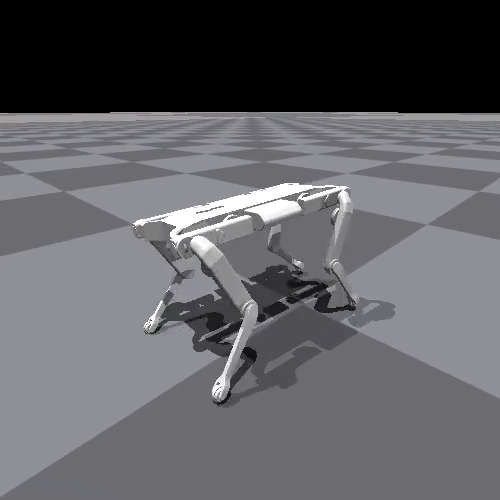}}\hspace{\gridspace}\color{cmiddle}\fbox{\includegraphics[scale=0.09]{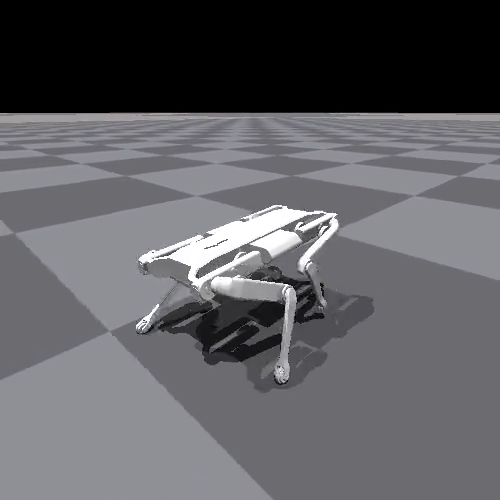}}} &
    
    \includegraphics[width=\scale\linewidth]{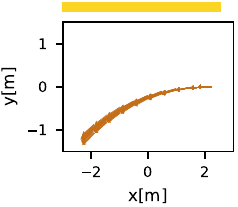} &
        
    \includegraphics[width=\scale\linewidth]{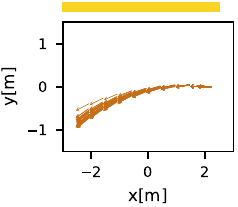} &
        
    \includegraphics[width=\scale\linewidth]{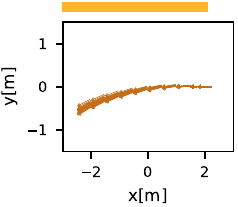} \\
        
    \includegraphics[width=\scale\linewidth]{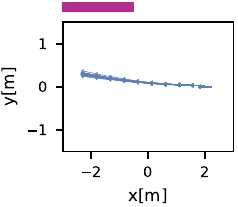} &
        
    \includegraphics[width=\scale\linewidth]{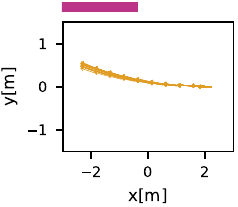}  &
        
    \includegraphics[width=\scale\linewidth]{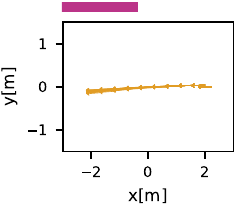} &
        
    \includegraphics[width=\scale\linewidth]{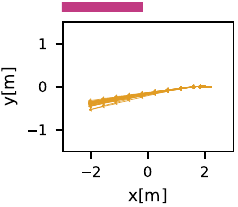} \\
        
    \raisebox{7ex}{\color{clow}\fbox{\includegraphics[scale=0.09]{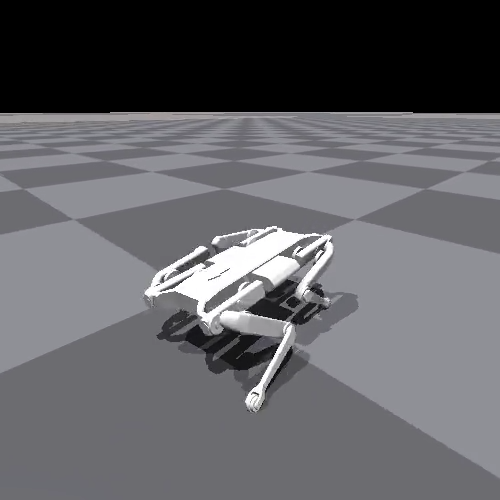}}\hspace{\gridspace}\color{cmiddle}\fbox{\includegraphics[scale=0.09]{figures/video/videoframe_zoom_diverse_0.png}}} &

    \includegraphics[width=\scale\linewidth]{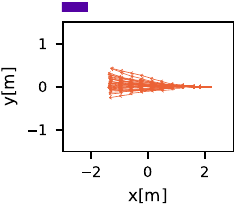} &
        
    \includegraphics[width=\scale\linewidth]{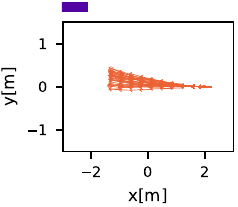} &
    
    \includegraphics[width=\scale\linewidth]{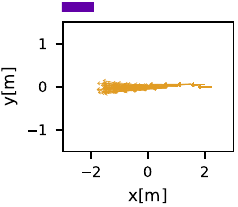} \\
    \end{tabular}
    \caption{A performance benchmark of skills learned in the locomotion task.
        The {\color{ourdarkblue} SMODICE-expert} walks with constant linear and angular velocity, and with {\color{cmiddle} middle} base-height.
        The learned skills recovers all base-height movements ({\color{clow} low}, {\color{cmiddle} middle}, {\color{chigh} high}) and exhibit different angular velocity.
        The colored horizontal bar at the top of each skill plot indicates the corresponding average base-height.}
    \label{fig:locom-robustness-experiment}
\end{figure}

\newlength{\imagewidth}
\setlength{\imagewidth}{0.2\linewidth}
\begin{figure}
    \newcommand\scale{0.2}
    \begin{tabular}{c c c c}
    \includegraphics[width=0.252\linewidth]{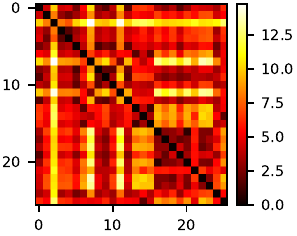} &
    \includegraphics[width=\scale\linewidth]{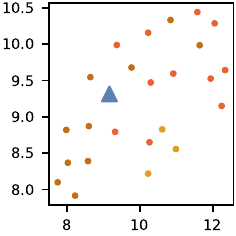} &
    \includegraphics[width=0.252\linewidth]{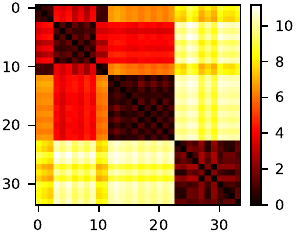} &
    \includegraphics[width=\scale\linewidth]{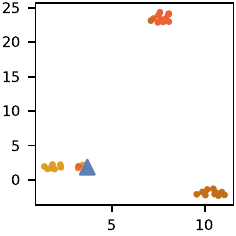} \\

    (a)\,\,\,\,\,\,\,\,\, & \,\,\,\,\,\,\,\,\,\,(b) & (c)\,\,\,\,\,\,\,\,\,\,\, & \,\,\,\,\,\,\,\,\,(d)
    \end{tabular}
    \caption{Tasks: (a,b) Locomotion and (c,d) Navigation.
        The triangle denotes the {\color{ourdarkblue} SMODICE-expert}, and the colored dots denote learned skills.
		(a,c) Successor features pair-wise $\ell_2$ distances between skills. The first row (column) is {\color{ourdarkblue} SMODICE-expert}, and all other rows (columns) are learned skills.
		(b,d) UMAP projection of successor features into 2D space. 
	}
	\label{fig:SFs}
\end{figure}

\newpage
\subsection{Obstacle Navigation Task}

\begin{wrapfigure}{r}{3.2cm}
	\vspace{-38pt}
	\includegraphics[width=3.2cm]{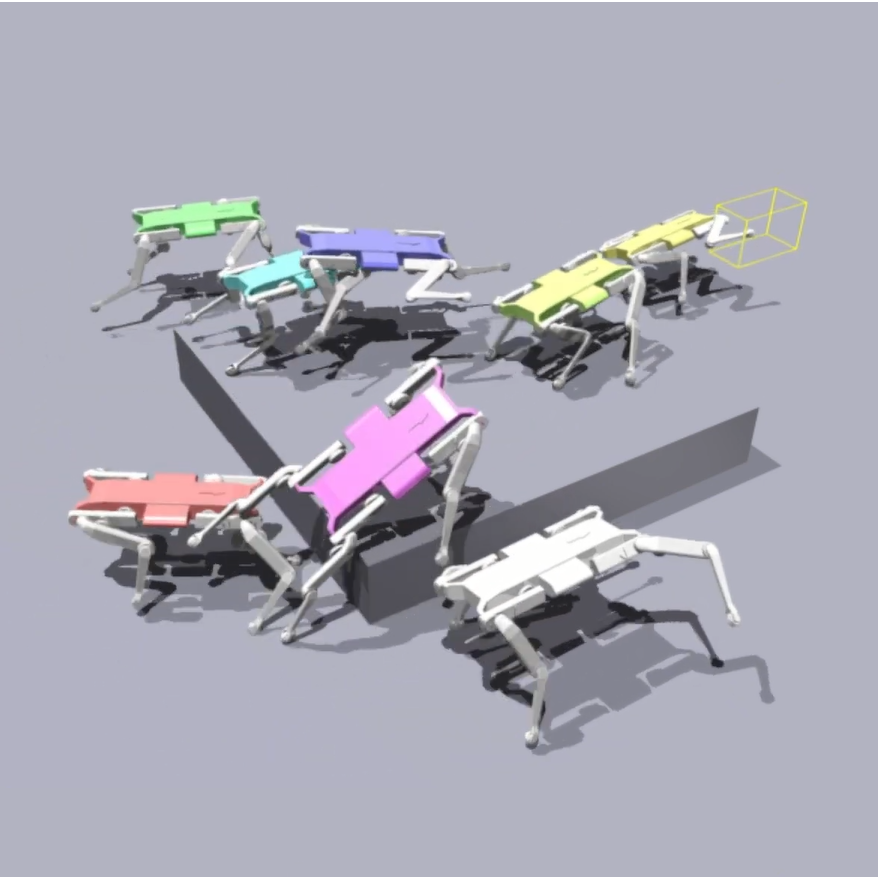}
\end{wrapfigure}

\textbf{Data Collection.} 
The expert dataset is collected from a multi-modal expert that, when initialized in front of a box, is trained to reach a target position behind the box by either going over the box or surrounding it from the left or the right side.
The offline dataset contains various non-expert behaviors collected at different time checkpoints during the expert's training procedure.
It is important to note that these behaviors may not reach the target position, nor do they have to remain standing for the entire episode.

\textbf{Diversity Objective.} 
The state space is projected onto the ``base linear velocity'' and ``base angular velocity'' components, each represented by a 3-dimensional vector.
This choice encourages diversity, as most trajectories in the offline dataset have similar body heights on the ground, but approach the obstacle from different directions and with different velocities.

\textbf{Results.}
These results answer \textbf{Q1}, \textbf{Q2}, and \textbf{Q3}. 
\Cref{fig:SFs}(c,d) addresses \textbf{Q1} by showing that the successor features of the learned skills are well separated: pairwise $\ell_2$ distances are large across many skill pairs, and the 2D UMAP projection reveals distinct clusters rather than a single concentrated mode.
This indicates that \textsc{Dual-Force} discovers genuinely different navigation strategies in the offline dataset.
\Cref{fig:box-learned-skills} addresses \textbf{Q2} by showing that these diverse skills still solve the task and remain consistent with the expert state occupancy: the learned policies reach the target position behind the box while covering all expert behaviors as well as additional modalities present in the offline data. 
Finally, \Cref{fig:box-robustness-experiment} addresses \textbf{Q3} by demonstrating that this diversity translates into robustness under adversarial obstacle variations.
When extra fence obstacles block the left side, the right side, or both sides, the learned skill set contains alternative behaviors that still reach the goal, with several skills outperforming the {\color{ourdarkblue} SMODICE-expert} in the one-sided blocking cases and matching it when both sides are blocked.

\begin{figure}
    \newcommand\scale{0.21}
    \begin{tabular}{c c c c}
        \includegraphics[width=\scale\linewidth]{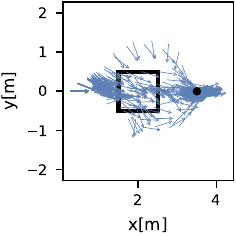} &
        
        \includegraphics[width=\scale\linewidth]{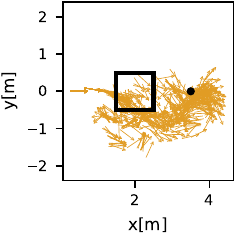} &
        
        \includegraphics[width=\scale\linewidth]{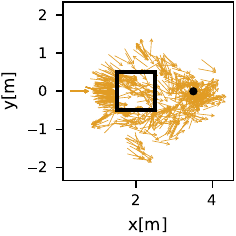} &
        
        \includegraphics[width=\scale\linewidth]{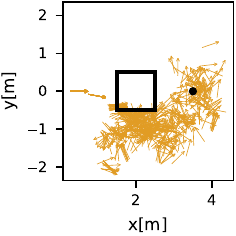} \\
        
        \includegraphics[width=\scale\linewidth]{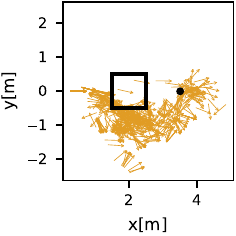} &

        \includegraphics[width=\scale\linewidth]{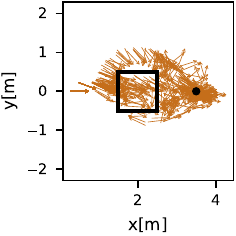} &

        \includegraphics[width=\scale\linewidth]{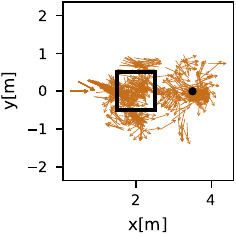} &
        
        \includegraphics[width=\scale\linewidth]{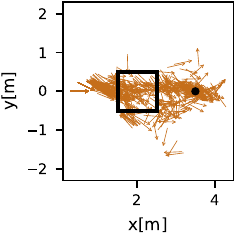} \\
        
        \includegraphics[width=\scale\linewidth]{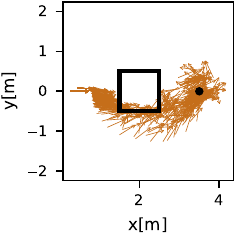} &

        \includegraphics[width=\scale\linewidth]{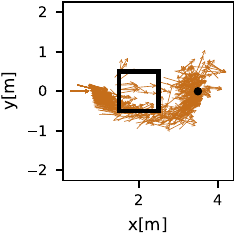} &
        
        \includegraphics[width=\scale\linewidth]{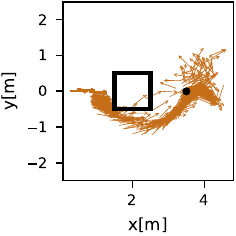} &
        
        \includegraphics[width=\scale\linewidth]{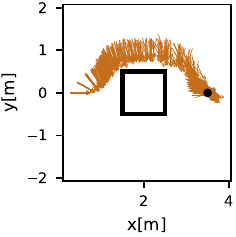} \\
        
        \includegraphics[width=\scale\linewidth]{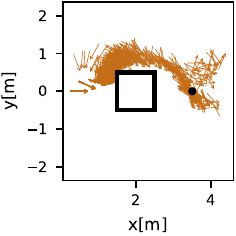} &

        \includegraphics[width=\scale\linewidth]{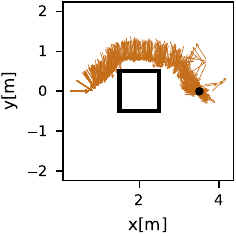} & 

        \includegraphics[width=\scale\linewidth]{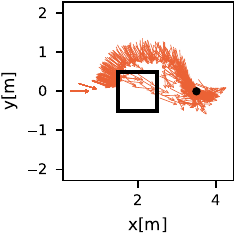} &

        \includegraphics[width=\scale\linewidth]{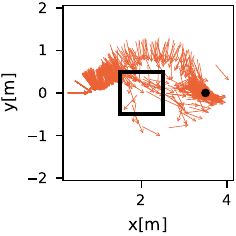} \\

        \includegraphics[width=\scale\linewidth]{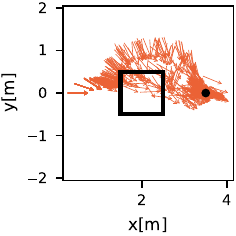} & 
        
        \includegraphics[width=\scale\linewidth]{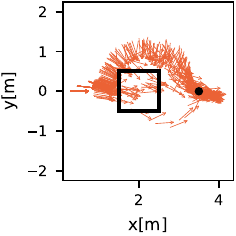} &
        
        \includegraphics[width=\scale\linewidth]{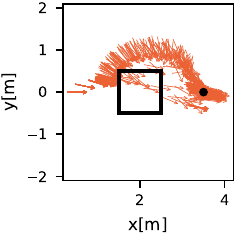} &
        
        \includegraphics[width=\scale\linewidth]{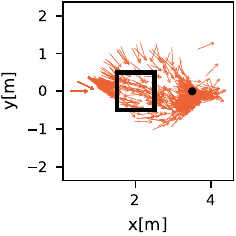} \\
        
        \includegraphics[width=\scale\linewidth]{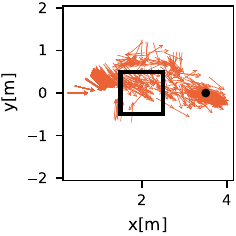} &
        
        \includegraphics[width=\scale\linewidth]{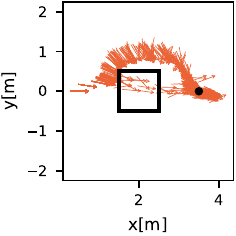} &
        
        \includegraphics[width=\scale\linewidth]{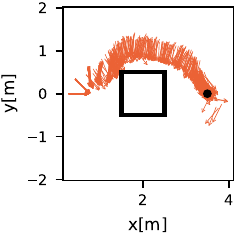} &
        
        \includegraphics[width=\scale\linewidth]{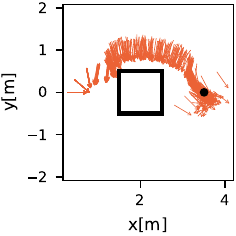} 
    \end{tabular}
	\caption{A performance benchmark of skills learned in the obstacle navigation task, where the {\color{ourdarkblue} SMODICE-expert} is initialized in front of a box of height 0.2m and reaches a target position behind the box.
    The learned skills exhibit diverse behaviors that cover various modalities of expert and offline datasets.}
	\label{fig:box-learned-skills}
\end{figure}

\begin{figure}
    \newcommand\scale{0.21}
    \begin{tabular}{cccc}
        \multicolumn{4}{c}{{\small (a) Left side is blocked}}\\
        \includegraphics[width=\scale\linewidth]{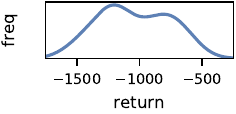} & 
    
        \includegraphics[width=\scale\linewidth]{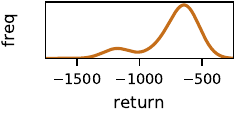} &
        
        \includegraphics[width=\scale\linewidth]{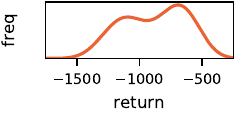} & 
    
        \includegraphics[width=\scale\linewidth]{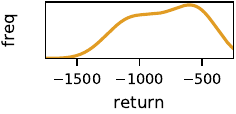} \\
        
        \includegraphics[width=\scale\linewidth]{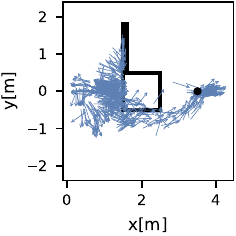} & 
        
        \includegraphics[width=\scale\linewidth]{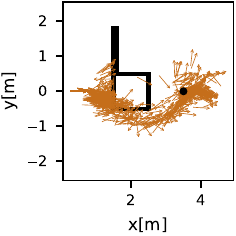} &
        
        \includegraphics[width=\scale\linewidth]{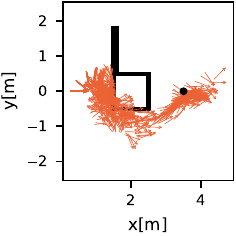} & 
        
        \includegraphics[width=\scale\linewidth]{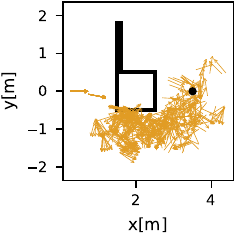} \\\\
        
        \multicolumn{4}{c}{{\small (b) Right side is blocked}}\\
        \includegraphics[width=\scale\linewidth]{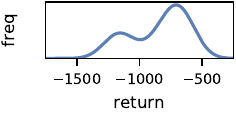} & 
        
        \includegraphics[width=\scale\linewidth]{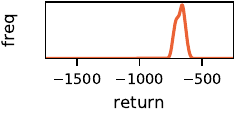} &
        
        \includegraphics[width=\scale\linewidth]{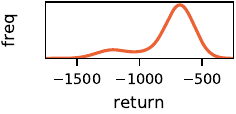} & 
        
        \includegraphics[width=\scale\linewidth]{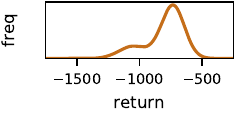} \\
        
        \includegraphics[width=\scale\linewidth]{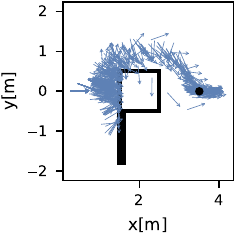} &
        
        \includegraphics[width=\scale\linewidth]{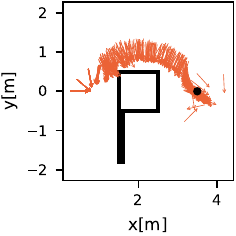} & 
        
        \includegraphics[width=\scale\linewidth]{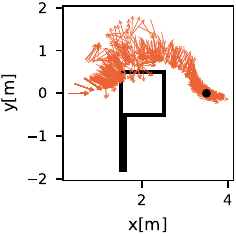} & 
        
        \includegraphics[width=\scale\linewidth]{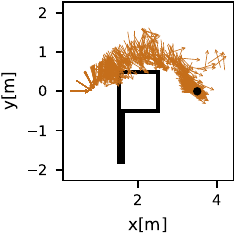} \\\\

        \multicolumn{4}{c}{{\small (c) Both left and right sides are blocked}}\\
        \includegraphics[width=\scale\linewidth]{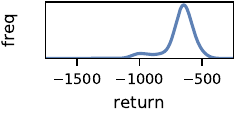} & 
        
        \includegraphics[width=\scale\linewidth]{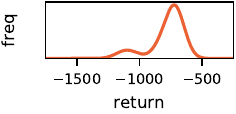} & 
        
        \includegraphics[width=\scale\linewidth]{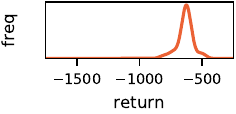} &
        
        \includegraphics[width=\scale\linewidth]{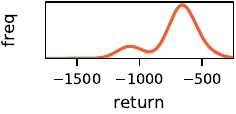} \\
        
        \includegraphics[width=\scale\linewidth]{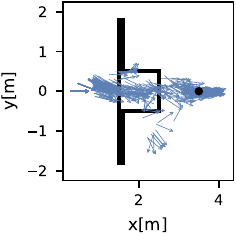} &
        
        \includegraphics[width=\scale\linewidth]{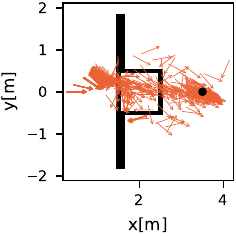} & 
        
        \includegraphics[width=\scale\linewidth]{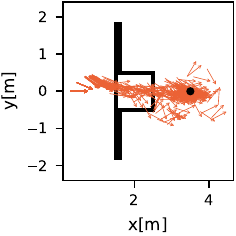} & 
        
        \includegraphics[width=\scale\linewidth]{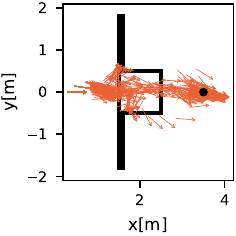} \\
        
    \end{tabular}
    \caption{A performance benchmark with additional fence obstacles, of height 0.6m, partially blocking the path from (a) the left side, (b) the right side, or (c) both the left and right sides.
		Among the diverse skills learned, there are several that outperform the {\color{ourdarkblue} SMODICE-expert} in (a,b) and perform on par with the {\color{ourdarkblue} SMODICE-expert} in (c).}
	\label{fig:box-robustness-experiment}
\end{figure}

\section{Related Work}\label{sec:rel-work}

Our work is situated within the broader landscape of skill discovery in RL, but it specifically addresses limitations prevalent in existing methods when applied to offline settings and the challenge of learning a large set of diverse skills. 
While significant progress has been made in unsupervised skill discovery algorithms in unconstrained and online environments \citep{eysenbach19diayn, campos2020edl, achiam2018valor, strouse2021disdain}, these approaches generally struggle to learn a substantial number of skills \citep{campos2020edl, achiam2018valor}. 
For instance, methods relying on skill predictability \citep{sharma2019dads} or ensemble-based information gain formulations \citep{strouse2021disdain} are inherently online, demanding extensive and often impractical interaction with the environment. 
Even offline methods based on skill predictability \citep{LiuZHW23} typically require expert demonstrations with labeled actions or merely extract skills already present in the dataset, failing to generate novel behaviors. 
This contrasts sharply with our approach, which overcomes these limitations by offering an offline algorithm that does not require a predefined number of skills or access to labeled expert demonstrations.

The challenge of unsupervised reinforcement learning, where agents learn meaningful representations and skills without explicit rewards, is central to our investigation. Seminal work on successor features \citep{dayan1993successor, barreto2017successor} demonstrated how decoupling environmental dynamics from rewards can facilitate transfer across tasks. 
While these representations have been combined with intrinsic motivation to enhance diversity \citep{GregorRW17vic, barreto2017successor, hansen20visr}, they often do not provide the explicit diversity signal we seek in our physically inspired VDW force objective. 
Our approach uses successor features as a fundamental building block for computing this diversity objective, but critically, we extend their utility to an off-policy estimation context, a capability not addressed by these prior works. 
This allows \textsc{Dual-Force} to efficiently estimate the VdW force components using data collected from any behavior policy.

Our work also aligns with the principles of Quality-Diversity (QD), a field focused on discovering diverse sets of optimal or near-optimal policies to improve the robustness and generalization of RL agents. 
Traditional QD algorithms, such as MAP-Elites \citep{CullyCTM15, MouretC15} and novelty search with local competition \citep{LehmanS11}, primarily employ evolutionary strategies to maintain policy collections, balancing quality and diversity \citep{Cully19, PughSS16, TaraporeCCM16}. 
While effective, these methods are typically online and computationally intensive, often requiring extensive simulation or real-world interaction. 
In contrast, our algorithm offers an off-policy solution to the QD problem, allowing for the discovery of diverse policies from offline datasets, thereby circumventing the high sample complexity that often limits the application of evolutionary QD methods in practical robotic scenarios.

Furthermore, \textsc{Dual-Force} advances constrained skill discovery in the offline setting.
Prior work considered value-constrained diversity maximization online, first with a single constraint \citep{zahavy2023discovering} and later with multiple constraints \citep{cheng2024dominic}, using the Van der Waals (VdW) force as the diversity signal. 
In parallel, \citet{vlastelica2024doi} studied offline diversity maximization under imitation constraints using a mutual-information objective. 
\textsc{Dual-Force} combines the strengths of these lines of work by bringing VdW-based diversity maximization into the offline imitation-constrained regime, thereby avoiding the skill discriminator required by mutual-information-based approaches.

Our method also builds directly on the DIstribution Correction Estimation (DICE) framework for off-policy learning \citep{nachum2020reinforcement}. 
DICE has been applied to off-policy policy gradients \citep{nachum2019algaedice}, offline imitation learning \citep{kim2022demodice, ma2022smodice}, and off-policy evaluation \citep{dai2020coindice}, as well as recent KL-regularized offline RL variants \citep{lee2021optidice, lee2022coptidice}. 
Our key contribution in this context is to use DICE not only for policy optimization, but also to estimate the quantities required by the VdW diversity objective offline, including successor features and dual-conjugate variables. 
This yields a single off-policy estimation pipeline for diversity optimization, constraint monitoring, and policy learning.

\noindent\textbf{Positioning.}
\textsc{Dual-Force} lies at the intersection of two lines of work: offline imitation-constrained skill discovery and online VdW-based constrained diversity maximization. 
Relative to the former, it replaces discriminator-based mutual-information objectives with a discriminator-free VdW objective that crucially admits off-policy estimation. 
Relative to the latter, it targets the offline regime and addresses the resulting non-stationary reward optimization via FRE-conditioned value and policy learning, while enabling latent-code-based skill recall.

\section{Conclusion}\label{sec:conc}

In this work, we introduced \textsc{Dual-Force}, an offline method for diversity maximization under state-only imitation constraints. 
By combining a discriminator-free off-policy estimator of the Van der Waals (VdW) diversity objective with FRE-conditioned value and policy learning, the method remains stable under the non-stationary rewards induced by alternating Lagrangian optimization and enables zero-shot recall of all skills encountered during training through stored latent codes.
Our results on \envsolo locomotion and obstacle navigation show that offline demonstration data can be transformed into a diverse set of robust behaviors without additional environment interaction. 
These findings suggest that diversity-driven offline imitation learning can be a practical route toward more adaptable robotic policies, especially in settings where online data collection is costly or unsafe.

\section{Acknowledgments}

We acknowledge the support from the German Federal Ministry of Education and Research (BMBF) through the Tübingen AI Center (FKZ: 01IS18039B). 
Georg Martius is a member of the Machine Learning Cluster of Excellence, funded by the Deutsche Forschungsgemeinschaft (DFG, German Research Foundation) under Germany’s Excellence Strategy – EXC number 2064/1 – Project number 390727645. 
This work was supported by the ERC - 101045454 REAL-RL. 
Pavel Kolev was supported by the Cyber Valley Research Fund and the Volkswagen Stiftung (No 98 571).

\newpage

\bibliography{references}
\bibliographystyle{abbrvnat}

\newpage
\appendix

\part*{Appendix}
\addcontentsline{toc}{section}{Appendix}
\setcounter{section}{0}
\renewcommand{\thesection}{\Alph{section}}

\renewcommand{\thetable}{S\arabic{table}}
\renewcommand{\thefigure}{S\arabic{figure}}
\renewcommand{\theequation}{S\arabic{equation}}
\setcounter{table}{0}
\setcounter{figure}{0}
\setcounter{equation}{0}

\section{Pseudocode of \textsc{Dual-Force}}\label{app:pseudo-code}

We summarize \textsc{Dual-Force} in \Cref{alg:dual-force}.
Our method combines five key ingredients: (i) a state-discriminator-based imitation objective for matching the expert state occupancy, 
(ii) a Van der Waals (VdW) force objective that provides a strong signal for learning diverse skills even in the offline setting,
(iii) DICE-based off-policy estimation of the quantities required to optimize the VdW objective from offline data, 
(iv) bounded Lagrange multipliers that balance diversity and imitation constraints, and 
(v) Functional Reward Encoding (FRE), which stabilizes learning under non-stationary rewards.
In addition, FRE provides a compact latent for each skill encountered during training, with the number of skills growing over iterations.

\begin{algorithm}[h]
	{\small
		\caption{\textsc{Dual-Force}}
		\label{alg:dual-force}
		\textbf{Input:} state-only expert dataset $\mathcal{D}_{E}\sim d_{E}(S)$; offline dataset $\mathcal{D}_O\sim d_O(S,A)$ such that $\mathcal{D}_{E}\subset\mathrm{States}[\mathcal{D}_{O}]$; $n$ number of VdW's state-action occupancies; $m$ number of subsets of states; $t$ number of state-reward pairs; Polyak scale $\alpha>0$\\
        \textbf{Output:} stored reward latents $\{z_i^k\}$ and the corresponding latent-conditioned policies $\{\pi_i(\cdot\mid\cdot,z_i^k)\}$ (recallable skills).\\
		\textbf{Initialize:} Sample $w_i^0$ uniformly at random from the probability simplex $\triangle^{|\mathcal{D}_O|}$, for all $i\in\{1,\dots,n\}$\\[-1ex]
		
        \textbf{Pre-train:} a state-discriminator $c^{\star}:\mathcal{S}\rightarrow(0,1)$ via optimizing the following objective with the gradient penalty in~\citep{gulrajani2017improved}
		$\max_{c}\mathbb{E}_{d_{E}(s)}[\log c(s)]+\mathbb{E}_{d_{O}(s)}[\log(1-c(s))]$\\
		\textbf{Pre-train:} a Functional Reward Encoding (FRE) $\mathcal{F} : (\mathcal{S}\times \mathcal{R})^m \mapsto \mathcal{Z}$ on state subsets of $\mathrm{States}[\mathcal{D}_{O}]$ and general unsupervised reward functions as described in \cref{app:FRE}\\[-1ex]
		
		\textbf{Repeat until convergence:}
		
		$\quad$\textbf{(Van der Waals Force)}
		
		$\quad$\textbf{For} each index $i\in\{1,\dots,n\}$:
		
		$\qquad$compute successor features $\psi_{i}^{k}:=\sum_{(s,a)\in\mathcal{D}_{O}}w_{i}^{k}(s,a)\phi(s)$
		
		$\qquad$compute closest distance $\ell_{i}^{k}:=\Vert\psi_{i}^{k}-\psi_{j_{i}^{\star}}^{k}\Vert_{2}$ where $j_{i}^{\star}:=\argmin_{j\neq i}\Vert\psi_{i}^{k}-\psi_{j}^{k}\Vert_{2}$
		
		$\qquad$compute VdW reward $\beta_{i}^{k}(s,a):=(1-(\ell_{i}^{k}/\ell_{0})^{3})\langle\phi(s),\psi_{i}^{k}-\psi_{j_{i}^{\star}}^{k}\rangle$
		
		$\qquad$compute reward $R_i^k(s,a):=(1-\sigma(\mu_{i}^{k}))\beta_{i}^{k}(s,a)+\sigma(\mu_{i}^{k})\log\frac{c^{\star}(s)}{1-c^{\star}(s)}$
		
		$\qquad$compute the mean $z_i^k$ over FREs $\{z_i^k(S_j)=\mathcal{F}(L(R_i^k,S_j))\}_{j=1}^{m}$, where $S_j \sim \mathrm{States}[\mathcal{D}_O]$ with $|S_j|=t$\\[-1ex]

		$\quad$\textbf{(Value Function and Policy)}
		
		$\quad$\textbf{For} each index $i\in\{1,\dots,n\}$:
		
		$\quad$$\quad$ update with GD the FRE-cond. value function $V_{i}(\cdot, z_i^k)$ optimizing \cref{eq:Vstar} with the reward $R_{i}^{k}$
		
		$\quad$$\quad$ compute ratios $w_i(s,a):=\mathrm{softmax}_{\mathcal{D}_O}\left(R_{i}^{k}(s,a)+\gamma\mathcal{T}V_{i}(s, a, z_i^k)-V_{i}(s, z_i^k)\right)$ for all $s,a\in\mathcal{D}_O$
		
		$\quad$$\quad$ compute Polyak average $w_{i}^{k+1}:=(1-\alpha)w_{i}^{k}+\alpha w_{i}$

		$\quad$$\quad$ update with GD the FRE-cond. policy $\pi_{i}(\cdot|\cdot, z_i^k)$ minimizing $\sum_{(s,a)\in\mathcal{D}_{O}}w_{i}^{k+1}(s,a)\log\pi_{i}(a|s, z_i^k)$\\[-1ex]
		
		$\quad$\textbf{(Bounded Lagrange Multipliers)}
		
		$\quad$\textbf{For} each index $i\in\{1,\dots,n\}$:
		
		$\qquad$compute an estimator 
		$\phi_{i}:=\log|\mathcal{D}_{O}|+\sum_{(s,a)\in\mathcal{D}_{O}}w_{i}^{k+1}(s,a)\Big[\log w_{i}^{k+1}(s,a)-\log\frac{c^{\star}(s)}{1-c^{\star}(s)}\Big]$
		
		$\quad$Update with GD $\mu^{k+1}$ minimizing the loss $\sum_{i=1}^{n}\sigma(\mu_{i}^{k})(\varepsilon-\phi_{i})$
	}
\end{algorithm}

\section{Reproducibility}

For the implementation of \textsc{Dual-Force} we used the PyTorch Autograd framework. 
The offline datasets were collected by \citet{vlastelica2024doi} and we used Isaac Gym to evaluate the learned skills.
The training was performed on an NVIDIA GeForce RTX 2080 Ti graphics card, and computations took in real-time:
\vspace{-8pt}
\begin{itemize}
    \setlength{\itemsep}{2pt} 
    \setlength{\parskip}{2pt} 
    \item FRE transformer (see \cref{app:FRE}): Locom (10h, batch 2048), Obstacle-Navi (10h, batch 1280)
    \item Skill-discriminator + {\color{ourdarkblue} SMODICE-expert}: Locom (0.3h, batch 8192), Obstacle-Navi (1h, batch 8192)
    \item \textsc{Dual-Force}: Locom (0.5h, batch 8192), Obstacle-Navi (1h, batch 2560)
\end{itemize}
\vspace{-8pt}
The \envsolo robot is developed as part of the Open Dynamic Robot Initiative~\citet{grimminger2020open}.

\newpage

\section{Imitation Constraint Relaxation}\label{app:sec:imitation-constraint-relax}

Our analysis makes use of the following assumption.

\begin{asm}[Expert coverage]\label{asm:base}
	We assume that $d_E(s)>0$ implies $d_O(s)>0$.
\end{asm}

\begin{lem}[State-only KL Estimator]\label{lem:state-KL-est}
	Under \Cref{asm:base}, we have
	\begin{equation}
	D_{\mathrm{KL}}\left(d_{i}(S)||d_{E}(S)\right)\leq-\mathbb{E}_{d_{i}(s)}\left[\log\frac{d_{E}(s)}{d_{O}(s)}\right]+D_{\mathrm{KL}}\left(d_{i}(S,A)||d_{O}(S,A)\right)\label{eq:kl-cnstr}
	\end{equation}
\end{lem}
\begin{proof}
	The statement follows by combining Claim~\ref{clm:one} and \ref{clm:two}.
\end{proof}

\begin{cor}[Structural]\label{cor:kl-estim}
	Under \Cref{asm:base}, the RHS of \cref{eq:kl-cnstr} is estimated by
	\[
	\mathbb{E}_{d_{O}(s,a)}\left[\eta_{i}(s,a)\left(\log\eta_{i}(s,a)-\log\frac{c^{\star}(s)}{1-c^{\star}(s)}\right)\right].
	\]
\end{cor}
\begin{proof}
	The statement follows by combining Lemma~\ref{lem:state-KL-est} and Claim~\ref{clm:three}.
\end{proof}

\subsection{Useful Facts}

\begin{clm}\label{clm:one} 
	It holds that
	\[
	D_{\mathrm{KL}}\left(d_{\pi_{1}}(S,A)||d_{\pi_{2}}(S,A)\right)=D_{\mathrm{KL}}\left(d_{\pi_{1}}(S)||d_{\pi_{2}}(S)\right)+\mathbb{E}_{d_{\pi_{1}}(s)}D_{\mathrm{KL}}\left(\pi_{1}(\cdot|s)||\pi_{2}(\cdot|s)\right)
	\]
\end{clm}
\begin{proof}
	We have
	\begin{eqnarray*}
		D_{\mathrm{KL}}\left(d_{\pi_{1}}(S,A)||d_{\pi_{2}}(S,A)\right) & = & \mathbb{E}_{d_{\pi_{1}}(s,a)}\left[\log\frac{d_{\pi_{1}}(s,a)}{d_{\pi_{2}}(s,a)}\right]=\mathbb{E}_{d_{\pi_{1}}(s,a)}\left[\log\frac{d_{\pi_{1}}(s)\pi_{1}(a|s)}{d_{\pi_{2}}(s)\pi_{2}(a|s)}\right]\\
		& = & \mathbb{E}_{d_{\pi_{1}}(s,a)}\left[\log\frac{d_{\pi_{1}}(s)}{d_{\pi_{2}}(s)}\right]+\mathbb{E}_{d_{\pi_{1}}(s)}\mathbb{E}_{\pi_{1}(a|s)}\left[\log\frac{\pi_{1}(a|s)}{\pi_{2}(a|s)}\right]\\
		& = & D_{\mathrm{KL}}\left(d_{\pi_{1}}(S)||d_{\pi_{2}}(S)\right)+\mathbb{E}_{d_{\pi_{1}}(s)}D_{\mathrm{KL}}\left(\pi_{1}(\cdot|s)||\pi_{2}(\cdot|s)\right)
	\end{eqnarray*}
\end{proof}

\begin{clm}\label{clm:two} 
	Under \Cref{asm:base}, we have
	\[
	D_{\mathrm{KL}}\left(d_{i}(S)||d_{E}(S)\right)=-\mathbb{E}_{d_{i}(s)}\left[\log\frac{d_{E}(s)}{d_{O}(s)}\right]+D_{\mathrm{KL}}\left(d_{i}(S)||d_{O}(S)\right)
	\]
\end{clm}
\begin{proof}
	We have
	\begin{eqnarray*}
		D_{\mathrm{KL}}\left(d_{i}(S)||d_{E}(S)\right) & = & \mathbb{E}_{d_{i}(s)}\left[\log\frac{d_{i}(s)}{d_{E}(s)}\right]\\
		& = & \mathbb{E}_{d_{i}(s)}\left[\log\frac{d_{i}(s)}{d_{O}(s)}\cdot\frac{d_{O}(s)}{d_{E}(s)}\right]\\
		& = & -\mathbb{E}_{d_{i}(s)}\left[\log\frac{d_{E}(s)}{d_{O}(s)}\right]+D_{\mathrm{KL}}\left(d_{i}(S)||d_{O}(S)\right)
	\end{eqnarray*}
\end{proof}

\begin{clm}\label{clm:three} 
	Let $\eta_{i}(s,a)=\frac{d_{i}(s,a)}{d_{O}(s,a)}$ for all $(s,a)\in\mathcal{D}_{O}$, and $c^{\star}(s)=\frac{d_{E}(s)}{d_{E}(s)+d_{O}(s)}$ for all $s\in\mathcal{D}_{E}\cup\mathcal{D}_{O}$ 
	\[
	\mathbb{E}_{d_{i}(s)}\left[\log\frac{d_{E}(s)}{d_{O}(s)}\right]\approx\mathbb{E}_{d_{O}(s,a)}\left[\eta_{i}(s,a)\log\frac{c^{\star}(s)}{1-c^{\star}(s)}\right]
	\]
\end{clm}

\begin{proof}
	We have
	\begin{eqnarray*}
		\mathbb{E}_{d_{i}(s)}\left[\log\frac{d_{E}(s)}{d_{O}(s)}\right] & = & \mathbb{E}_{d_{i}(s)}\mathbb{E}_{\pi(a|s)}\left[\log\frac{d_{E}(s)}{d_{O}(s)}\right]=\mathbb{E}_{d_{i}(s,a)}\left[\log\frac{d_{E}(s)}{d_{O}(s)}\right]\\
		& \approx & \mathbb{E}_{d_{O}(s,a)}\left[\eta_{i}(s,a)\log\frac{c^{\star}(s)}{1-c^{\star}(s)}\right]
	\end{eqnarray*}
\end{proof}

\section{Off-policy KL Estimator}\label{app:sec:state-kl-estimate}
Recall that the weight $w_{i}(s,a)$ is defined w.r.t. a fixed dataset
$\mathcal{D}_{0}$ and reads
\[
w_{i}(s,a)=\mathrm{softmax}_{\mathcal{D}_{O}}(\delta_{i}(s,a))=\frac{\exp\{\delta_{i}(s,a)\}}{\sum_{(s^{\prime},a^{\prime})\in\mathcal{D}_{O}}\exp\{\delta_{i}(s^{\prime},a^{\prime})\}},
\]
where the TD error $\delta_{i}(s,a)=R_{i}^{k}(s,a)+\gamma\mathcal{T}V_{i}^{\star}(s,a)-V_{i}^{\star}(s)$.
In contrast, the importance ratio $\eta_{i}(s,a)$ is defined in terms of the expectation of the state-action occupancy $d_{O}$, namely 
\[
\eta_{i}(s,a)=\mathrm{softmax}_{d_{O}(s,a)}(\delta_{i}(s,a))=\frac{\exp\{\delta_{i}(s,a)\}}{\mathbb{E}_{d_{O}(s^{\prime},a^{\prime})}\exp\{\delta_{i}(s^{\prime},a^{\prime})\}}.
\]

\begin{clm}\label{clm:KL-estimator}
	Given an offline dataset $\mathcal{D}_{O}$ sampled
	u.a.r. from state-action occupancy $d_{O}$, an estimator of the importance
	ratio $\eta_{i}(s,a)$ is given by $\widetilde{\eta}_{i}(s,a):=|\mathcal{D}_{O}|w_{i}(s,a)$.    
\end{clm}
\begin{proof}
	Combining $\frac{1}{|\mathcal{D}_{O}|}\sum_{(s^{\prime},a^{\prime})\in\mathcal{D}_{O}}\exp\{\delta_{i}(s^{\prime},a^{\prime})\}$
	is an estimator of the expectation $\mathbb{E}_{d_{O}(s^{\prime},a^{\prime})}\exp\{\delta_{i}(s^{\prime},a^{\prime})\}$
	and the definition of weight $w_{i}(s,a)$ we have 
	\begin{eqnarray*}
		\eta_{i}(s,a)&=&\frac{\exp\{\delta_{i}(s,a)\}}{\mathbb{E}_{d_{O}(s^{\prime},a^{\prime})}\exp\{\delta_{i}(s^{\prime},a^{\prime})\}}\\&\approx&\frac{\exp\{\delta_{i}(s,a)\}}{\frac{1}{|\mathcal{D}_{O}|}\sum_{(s^{\prime},a^{\prime})\in\mathcal{D}_{O}}\exp\{\delta_{i}(s^{\prime},a^{\prime})\}}=|\mathcal{D}_{O}|w_{i}(s,a)=\widetilde{\eta}_{i}(s,a).    
	\end{eqnarray*}
\end{proof}

\begin{lem}\label{lem:KL-estimator}
	The KL-divergence $D_{\mathrm{KL}}(d_{i}(S,A)||d_{O}(S,A))$
	admits the following off-policy estimator
	\[
	\log|\mathcal{D}_{O}|\,\,+\,\sum_{(s,a)\in\mathcal{D}_{O}}w_{i}(s,a)\log w_{i}(s,a).    
	\]
\end{lem}
\begin{proof}
	Combining the definition of $\eta_{i}(s,a)=d_{i}(s,a)/d_{O}(s,a)$ and Claim~\ref{clm:KL-estimator}, we have
	\begin{eqnarray*}
		D_{\mathrm{KL}}\left(d_{i}(S,A)||d_{O}(S,A)\right) & = & \mathbb{E}_{d_{i}(s,a)}\log\eta_{i}(s,a)\\
		& = & \mathbb{E}_{d_{O}(s,a)}\eta_{i}(s,a)\log\eta_{i}(s,a)\\
		& \approx & \frac{1}{|\mathcal{D}_{O}|}\sum_{(s,a)\in\mathcal{D}_{O}}\widetilde{\eta}_{i}(s,a)\log\widetilde{\eta}_{i}(s,a)\\
		& = & \sum_{(s,a)\in\mathcal{D}_{O}}w_{i}(s,a)\log\left(|\mathcal{D}_{O}|w_{i}(s,a)\right)\\
		& = & \log|\mathcal{D}_{O}|\,\,+\,\sum_{(s,a)\in\mathcal{D}_{O}}w_{i}(s,a)\log w_{i}(s,a)
	\end{eqnarray*}
\end{proof}

\section{Successor Features as Diversity Measure}

\begin{lem}[Convex Diversity Objective]\label{lem:div_is_conv}
	Let $\Phi\in\mathbb{R}^{d\times(S\times A)}$ be a feature map and $d_{i}\in\triangle^{S\times A}$ be a probability distribution.
	Then for the feature vector $\psi_{i}=\Phi d_{i}\in\mathbb{R}^{d}$ we have 
	\[
	\nabla_{d_{i}}\frac{1}{2}\Vert\psi_{i}-\psi_{j}\Vert_{2}^{2}=\Phi^{T}(\psi_{i}-\psi_{j}).
	\]
	Further, the corresponding Hessian is positive semi-definite matrix, i.e.,
	\[
	\nabla_{d_{i}}\Phi^{T}\Phi(d_{i}-d_{j})=\Phi^{T}\Phi\succeq0.
	\]
	In particular, $\frac{1}{2}\Vert\Phi d_{i}-\Phi d_{j}\Vert_{2}^{2}$ is a convex function w.r.t. $d_{i}$.
\end{lem}
\begin{proof}
	Observe that
	\begin{eqnarray*}
		\nabla_{d_{i}(s,a)}\frac{1}{2}\sum_{\ell=1}^{n}(\Phi_{\ell,:}d_{i}-\Phi_{\ell,:}d_{\pi_{2}})^{2} & = & \sum_{\ell=1}^{n}(\Phi_{\ell,:}d_{i}-\Phi_{\ell,:}d_{j})[\phi(s,a)]_{\ell}\\
		& = & \left(\sum_{\ell=1}^{n}\Phi_{\ell,:}[\phi(s,a)]_{\ell}\right)(d_{i}-d_{j})\\
		& = & \phi(s,a)^{T}\Phi(d_{i}-d_{j})\\
		& = & \phi(s,a)^{T}(\psi_{i}-\psi_{j})
	\end{eqnarray*}
	Hence, we have
	\begin{eqnarray*}
		\nabla_{d_{i}}\frac{1}{2}\Vert\psi_{i}-\psi_{j}\Vert_{2}^{2} & = & \nabla_{d_{i}}\frac{1}{2}\Vert\Phi d_{i}-\Phi d_{j}\Vert_{2}^{2}\\
		& = & \nabla_{d_{i}}\frac{1}{2}\sum_{\ell=1}^{n}(\Phi_{\ell,:}d_{i}-\Phi_{\ell,:}d_{j})^{2}\\
		& = & \sum_{\ell=1}^{n}\Phi_{\ell,:}(d_{i}-d_{j})\Phi_{\ell,:}^{T}\\
		& = & \left[\sum_{\ell=1}^{n}\Phi_{\ell,:}^{T}\Phi_{\ell,:}\right](d_{i}-d_{j})\\
		& = & \Phi^{T}\Phi(d_{i}-d_{j})\\
		& = & \Phi^{T}(d_{i}-d_{j})
	\end{eqnarray*}
	and
	\[
	\nabla_{d_{i}}\Phi^{T}\Phi(d_{i}-d_{j})=\Phi^{T}\Phi.
	\]
\end{proof}

\section{Training Metrics}\label{app:diversity_constraint}

\begin{figure}[H]
    \setlength{\gridspace}{1pt}
    \newcommand\scale{0.215}
    \setlength{\fboxrule}{1pt}
    \begin{tabular}{cccc}
    
    \multicolumn{4}{c}{\textbf{Locomotion Task}} \\[0.5em] 
    
    \includegraphics[width=\scale\linewidth]{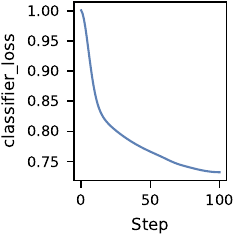} &
    
    \includegraphics[width=\scale\linewidth]{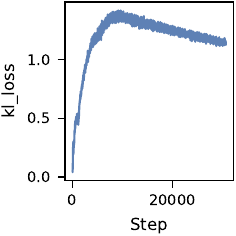} &
    
    \includegraphics[width=\scale\linewidth]{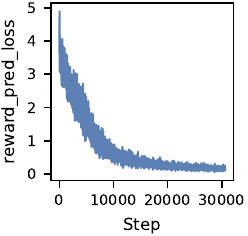} &
    
    \includegraphics[width=\scale\linewidth]{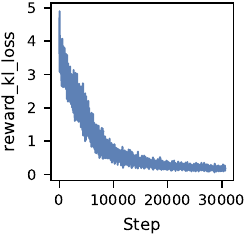} \\

    \includegraphics[width=\scale\linewidth]{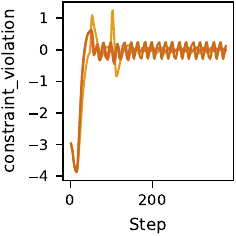} &
    
    \includegraphics[width=\scale\linewidth]{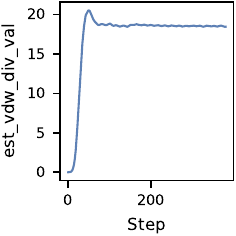} & 
    
    \includegraphics[width=\scale\linewidth]{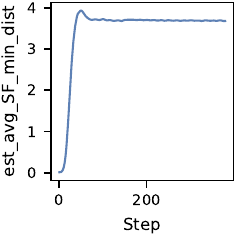} &
    
    \includegraphics[width=\scale\linewidth]{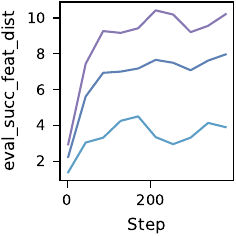} \\[0.5em]
    
    \multicolumn{4}{c}{\textbf{Obstacle Navigation Task}} \\[0.5em] 
    
    \includegraphics[width=\scale\linewidth]{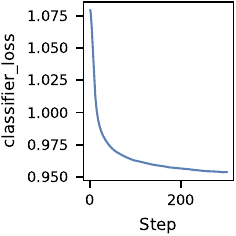} &
    
    \includegraphics[width=\scale\linewidth]{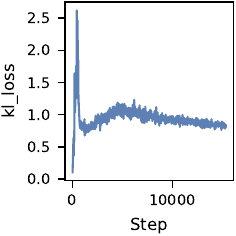} &
    
    \includegraphics[width=\scale\linewidth]{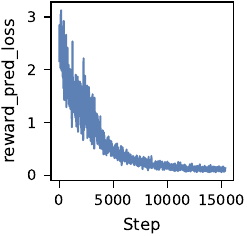} &
    
    \includegraphics[width=\scale\linewidth]{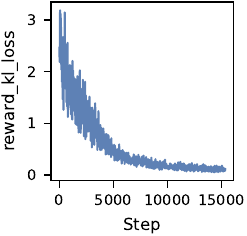} \\

    \includegraphics[width=\scale\linewidth]{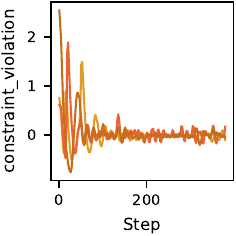} &
    
    \includegraphics[width=\scale\linewidth]{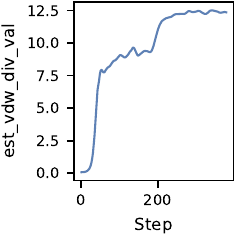} & 

    \includegraphics[width=\scale\linewidth]{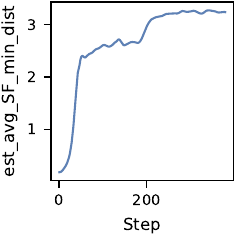} &
    
    \includegraphics[width=\scale\linewidth]{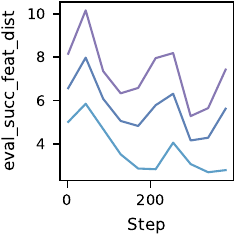} \\

    \end{tabular}
    \caption{Training metrics of State Discriminator, Functional Reward Encoding, and \textsc{Dual-Force} algorithm.}
    \label{fig:locom_box_wnb}
\end{figure}

In \cref{fig:locom_box_wnb}, we report important metrics computed during the training of the state-discriminator $c^{\star}$, the Functional Reward Encoding $\mathcal{F}$, and the \textsc{Dual-Force} algorithm.
In particular, the four subplots in the first row show:
(1) a training loss of the state-discriminator $c^{\star}$; and (2,3,4) training losses of the Functional Reward Encoding.
The subplots in the second row show four evaluation metrics of the \textsc{Dual-Force} algorithm:
(5) an estimated constraint violations of the three state-action occupancies {\color{ourorange} $d_1$}, {\color{ourbrown} $d_2$}, and {\color{ourred} $d_3$}; 
(6,7) Van der Waals diversity measures -- an estimated VdW objective value (see \cref{eq:vdw_force}) and an estimated average of minimum $\ell_2^2$ distance between successor features of distinct {\color{ourorange} $d_1$}, {\color{ourbrown} $d_2$}, and {\color{ourred} $d_3$} (see \cref{eq:rep_force}); and (8) an online evaluation of $\{$minimum, average, maximum$\}$ successor features $\ell_2^2$ distance among {\color{ourorange} $d_1$}, {\color{ourbrown} $d_2$}, and {\color{ourred} $d_3$}.

In the VdW diversity objective, the state space is projected onto different components depending on the task at hand (locomotion or obstacle navigation).
For the locomotion task, the state space is projected onto the joint position components, represented by a 12-dimensional vector (see \cref{app:solo12_state_space}), which serves as a proxy for body height, as this information is missing from the offline dataset.
For the obstacle navigation task, the state space is projected onto the ``base linear velocity'' and ``base angular velocity'' components, each represented by a 3-dimensional vector, as most trajectories in the offline dataset have similar body heights on the ground, but approach the obstacle from different directions and with different velocities.

\textbf{Locomotion.}
The offline dataset is composed of 248,751 expert and 996,000 behavior transitions, in total 1,244,751.
The state space has 48-dimensions, the VdW diversity objective projects the state space onto the ``joint positions'' (12-dimensional vector), the target minimum distance between SFs of learned skills is $\ell_0 = 6.0$, and the target imitation constraint threshold is $\epsilon = 4.0$.

\textbf{Obstacle Navigation.}
The offline dataset is composed of 654,501 expert and 749,501 behavior transitions, in total 1,404,002.
The state space has 171-dimensions, the VdW diversity objective projects the state space onto the ``base linear velocity'' and ``base angular velocity'' (each with 3-dimensions), the target minimum distance between SFs of learned skills is $\ell_0 = 4.0$, and target imitation constraint threshold is $\epsilon = 1.0$.

\section{\envsolo State Space}\label{app:solo12_state_space}

For fair comparison and consistency, in terms of quality and diversity of learned skills, our experiments closely follow the setup in \citep{vlastelica2024doi} and use offline datasets whose collection process is described in Section G ``SOLO-12 Dataset Collection'' of their work.

\textbf{State Space.} 
For the locomotion task, the state space has 48 dimensions:

$\,\,$ [3-dims each] ``base linear velocity'', ``base angular velocity'', ``projected gravity'', ``commanded velocity'';

$\,\,$ [12-dims each] ``joint positions'', ``joint velocity'', ``previous action''.

The ``joint positions'' components capture the movements of each of the four legs: front left, front right, hind left, and hind right - and includes hip abduction/adduction, hip flexion/extension, and knee flexion/extension. 
The ``joint velocity'' components have similar semantics to the ``joint positions''.
The ``previous action'' encodes the target joint positions from the previous time step.

For the obstacle navigation task, the state space contains 171 dimensions:
\vspace{-8pt}
\begin{itemize}
    \setlength{\itemsep}{3pt} 
    \setlength{\parskip}{3pt} 
    \item the above locomotion state with 48 dimensions;
    \item ``surrounding height map'' of the robot with 121 dimensions;
    \item ``remaining time'' until the end of the episode with 1 dimension.
\end{itemize}

\section{Pre-training of Functional Reward Encoding}\label{app:FRE}

We pretrain the FRE model following the approach of \citet{frans24fre}.
For both the locomotion task and the obstacle navigation task, to ensure wide diversity of general unsupervised reward functions, we generate a list of rewards as follows:
\vspace{-8pt}
\begin{itemize}
    \setlength{\itemsep}{3pt} 
    \setlength{\parskip}{3pt} 
    \item[$30\times$] linear functions with random weights
    \item[$30\times$] two-layer perceptron (MLP) neural networks with random weights and hidden units in $[(128,64), (128,128), (256,128), (256,256), (512,256), (512,512)]$
    \item[$27\times$] combination of simple human-engineered rewards that incentivize constant linear and angular velocity in different directions, and different joint angle heights.
\end{itemize}
\vspace{-5pt}
It is important to note that the above FRE latent representation does not affect the diversity objective itself; it serves as an index that assigns a stable code to each non-stationary reward, enabling FRE-conditioned value/policy learning and skill recall.

\end{document}

%% file: math_commands.tex
\usepackage{amsmath,amsthm,amsfonts,amssymb,bm}

\def\eqref#1{equation~\ref{#1}}

\def\1{\bm{1}}

\DeclareMathAlphabet{\mathsfit}{\encodingdefault}{\sfdefault}{m}{sl}
\SetMathAlphabet{\mathsfit}{bold}{\encodingdefault}{\sfdefault}{bx}{n}

\newcommand{\E}{\mathbb{E}}

\DeclareMathOperator*{\argmin}{arg\,min}

\theoremstyle{plain}
\newtheorem{thm}{Theorem}[section]
\newtheorem{lem}[thm]{Lemma}

\newtheorem{cor}[thm]{Corollary}
\newtheorem{clm}[thm]{Claim}
\newtheorem{asm}[thm]{Assumption}

\theoremstyle{definition}

\theoremstyle{remark}

\newcommand{\Dkl}{\mathrm{D}_\mathrm{KL}}

\newcommand{\norm}[1]{\left\lVert#1\right\rVert}
\newcommand{\innerproduct}[2]{\langle #1, #2 \rangle}

%% file: header.tex
\usepackage{booktabs}       

\usepackage{graphics,color}

\usepackage{algorithm}
\usepackage{algorithmic}
\usepackage{enumitem}

\usepackage{xspace}
\makeatletter
\DeclareRobustCommand\onedot{\futurelet\@let@token\@onedot}
\def\@onedot{\ifx\@let@token.\else.\null\fi\xspace}
\makeatother

\usepackage{xr-hyper}
\makeatletter
\newcommand*{\addFileDependency}[1]{
  \typeout{(#1)}
  \@addtofilelist{#1}
  \IfFileExists{#1}{}{\typeout{No file #1.}}
}
\makeatother

\usepackage{xcolor}
\definecolor{ourblue}{rgb}{0.368,0.507,0.71}
\definecolor{ourorange}{rgb}{0.881,0.611,0.142}
\definecolor{ourgreen}{rgb}{0.56,0.692,0.195}
\definecolor{ourred}{rgb}{0.923,0.386,0.209}
\definecolor{ourviolet}{rgb}{0.528,0.471,0.701}
\definecolor{ourbrown}{rgb}{0.772,0.432,0.102}
\definecolor{ourlightblue}{rgb}{0.364,0.619,0.782}
\definecolor{ourdarkolive}{rgb}{0.572,0.586,0.}
\definecolor{ourpurple}{rgb}{0.528,0.471,0.701}

\definecolor{ourdarkbrown}{rgb}{0.618,0.348,0.0816}
\definecolor{ourcyan}{rgb}{0.364,0.619,0.782}
\definecolor{ourgrey}{rgb}{0.75,0.75,0.75}

\definecolor{ourdarkred}{rgb}{0.67, 0.22, 0.07}
\definecolor{ourdarkorange}{rgb}{0.71, 0.49, 0.1}
\definecolor{ourdarkblue}{rgb}{0.27, 0.4, 0.58}
\definecolor{ourdarkgreen}{rgb}{0.41, 0.51, 0.15}

\definecolor{ourcyan2}{rgb}{0.125,0.722,0.804}
\definecolor{ourred2}{rgb}{0.863,0.184,0.047}
\definecolor{ouryellow2}{cmyk}{0,0.16,1.0,0.07}
\definecolor{ourviolet2}{cmyk}{0.55,0.56,0,0.47}
\definecolor{ourorange2}{cmyk}{0,0.46,0.89,0.11}

\definecolor{chigh}{RGB}{251,209,36}
\definecolor{cmiddle}{RGB}{185,51,136}
\definecolor{clow}{RGB}{81,1,162}